%% file: main.tex
\documentclass[letterpaper]{article}
\usepackage{aaai25} %
\usepackage{times} %
\usepackage{helvet} %
\usepackage{courier} %
\usepackage[hyphens]{url} %
\usepackage{graphicx} %
\usepackage{graphicx}  %
\usepackage{natbib}  %
\usepackage{caption}  %
\usepackage{algorithm}
\usepackage{algorithmic}
\usepackage{newfloat}
\usepackage{listings}
\DeclareCaptionStyle{ruled}{labelfont=normalfont,labelsep=colon,strut=off} %
\floatstyle{ruled}
\newfloat{listing}{tb}{lst}{}
\floatname{listing}{Listing}

\usepackage{xspace}
\usepackage{ifthen}
\input{defs}

\usepackage{amsmath}
\usepackage{adjustbox}

\newtheorem{theorem}{Theorem}
\newtheorem{corollary}{Corollary}

\newenvironment{proof}{%
\paragraph{Proof}}{%
\hspace*{\fill} 	\noindent $\Box$\par\addvspace{\topsep}
}

\usepackage[algo2e,ruled,linesnumbered]{algorithm2e}

\newboolean{forAFclearance}
\setboolean{forAFclearance}{false}
\setboolean{clearedp}{true}
\setboolean{anonymous}{false}
\newboolean{supplements}
\setboolean{supplements}{true}
\newboolean{forArxiv}
\setboolean{forArxiv}{true}

\setboolean{draft}{false}

\usepackage{enumitem}

\title{\papertitle\ifthenelse{\boolean{forArxiv}}{
  }{\footnote{Figures with titles in the form of "Figure S\#" can be found in the supplemental materials. Supplemental content can be found at: <ARXIV LINK GO HERE>.}}}
\author{
    Paul Zaidins\textsuperscript{\rm 1},
    Robert P. Goldman\textsuperscript{\rm 2},
    Ugur Kuter\textsuperscript{\rm 2}, 
    Dana Nau\textsuperscript{\rm 1}, 
    Mark Roberts\textsuperscript{\rm 3}
}
\affiliations{
    \textsuperscript{\rm 1}Department of Computer Science and Institute for Systems Research, University of Maryland \\
     \textsuperscript{\rm 2}SIFT, LLC \\
     \textsuperscript{\rm 3}Navy Center for Applied Research in AI, Naval Research Laboratory, Washington, DC, USA \\
     pzaidins@umd.edu,
     rpgoldman@sift.net,
     ukuter@sift.net,
     nau@umd.edu,
     mark.c.roberts20.civ@us.navy.mil
     
}

\begin{document}
\pagestyle{plain}
\thispagestyle{plain}
\maketitle

\begin{abstract}
This paper provides theoretical and empirical comparisons of three recent hierarchical plan repair algorithms: \shopfixer, \ipyhopper, and \rewrite. Our theoretical results show that the three algorithms correspond to three different definitions of the plan repair problem, leading to differences in the algorithms’ search spaces, the repair problems they can solve, and the kinds of repairs they can make. Understanding these distinctions is important when choosing a repair method for any given application.

Building on the theoretical results, we evaluate the algorithms empirically in a series of benchmark planning problems. Our empirical results provide more detailed insight into the runtime repair performance of these systems and the coverage of the repair problems solved, based on algorithmic properties such as replanning, chronological backtracking, and backjumping over plan trees.
\end{abstract}

\input{intro}

\input{plan-repair-strategies-new}

\input{theory}

\input{experiments}

\input{results}

\input{discussion}

\input{conclusions}

\section*{Acknowledgements}
\ack

\clearpage{}

\bibliography{main,rpg-zotero}
\when{supplements}{%

\input{supplemental}}
\end{document}

%% file: defs.tex
\newcommand{\newabbrev}[2]{\newcommand{#1}{#2\xspace}}
\newboolean{draft}
\newcommand{\when}[2]{\ifthenelse{\boolean{#1}}{#2}{}}
\newcommand{\Unless}[2]{\ifthenelse{\boolean{#1}}{}{#2}}
\newcommand{\IfThenElse}[3]{\ifthenelse{\boolean{#1}}{#2}{#3}}
\renewcommand{\ifdraft}[1]{\ifthenelse{\boolean{draft}}{#1}{}}

\usepackage{color}

\newcommand{\hide}[1]{}

\newcommand{\newtext}[1]{\ifthenelse{\boolean{draft}}{\textcolor{blue}{#1}}{#1}}

\newabbrev{\papertitle}{HTN Plan Repair Algorithms Compared: Strengths and Weaknesses of Different Methods}
\newabbrev{\ipyhopper}{\textsc{IPyHOPPER}}
\newabbrev{\shopfixer}{\textsc{SHOPFixer}}
\newabbrev{\rewrite}{\textsc{Rewrite}}
\newabbrev{\shop}{\textsc{Shop3}}
\newabbrev{\ie}{\textit{i.e.}}
\newabbrev{\eg}{\textit{e.g.}}
\newabbrev{\etc}{\textit{etc.}}
\newabbrev{\etal}{\textit{et al.}}
\newabbrev{\hetal}{H\"{o}ller, \etal}
\newabbrev{\rsthree}{Rewrite-\shop{}}

\newabbrev{\ipshort}{IPH}
\newabbrev{\sfshort}{SF}
\newabbrev{\rwshort}{RW}

\newcommand{\replace}{\operatorname{replace}}
\newcommand{\plan}{\operatorname{plan}}
\newcommand{\pre}{\operatorname{pre}}
\newcommand{\prestar}{\operatorname{pre}^{*}}
\newcommand{\M}{\mathcal{M}}

\newboolean{clearedp}
\newboolean{anonymous}

\newcommand{\ndraft}[1]{}
\newcommand{\indraft}[1]{\ifthenelse{\boolean{draft}}{#1}{}}
\newcommand{\notdraft}[1]{\ifthenelse{\boolean{draft}}{}{#1}}

\newcommand{\disclaimer}{Any opinions, findings and conclusions, or 
recommendations expressed
in this material are those of the authors and do not necessarily
reflect the views of the \contractingagency.}
\newcommand{\contractingagency}{AFRL}

\newcommand{\ack}{\ifthenelse{\boolean{clearedp} \AND
    \not{\boolean{anonymous}}}{This project is sponsored by the Air
    Force Research Laboratory (AFRL) under contract FA8750- 23-C-0515
    for the HI-DE-HO STTR Phase 2 program.  Distribution Statement A.
    Approved for public release: distribution is unlimited.
    \disclaimer{}
    Thanks to the anonymous
    reviewers for their helpful feedback. MR thanks ONR and NRL for funding portions of his research.}{Submitted for public approval.}}

\usepackage{amssymb}

%% file: intro.tex
\section{Introduction}\label{sec:introduction}

Dynamic control and management of plans during execution is challenging in many multi- and hierarchically-organized agent systems, e.g., military operations, warehouse automation, and other practical applications. During execution, agents' plans may fail due to precondition failures induced by exogenous events or by actions of other agents. \citet{fox06plan} showed that in the face of disruptions, plan repair could provide new plans faster and with fewer revisions than replanning from scratch. They used the term ``stability'' to refer to the new plan's similarity to the old one, by analogy to the term from control theory.

Agent systems using Hierarchical Task Network (HTN) planning generate by decomposing high-level (more abstract) tasks into lower-level (less abstract) subtasks before or during execution. Generalizing stable plan repair from classical planning to HTN planning requires localization of disruptions, errors and failures in the hierarchies, and it uses problem refinement methods that take advantage of such localizations to provide better stability \cite{GoldmanETAL:PlanRepair:2020}. 
Early work on hierarchical plan repair introduced validation graphs to hierarchical partial-order planning, using the validation graphs to identify disruptions and make patches in the partial-order plans \cite{kambhampati92priar}.
Several recent algorithms, namely \shopfixer \cite{GoldmanETAL:PlanRepair:2020}, \ipyhopper \cite{Zaidins:HPlan2023}, and an unnamed algorithm that we will call \rewrite \cite{HollerHTNPlanRepair2020}, have developed state-of-the art plan-repair strategies. 
They build on several previous methods \cite{ayan07hotride,kuter12shoplifter,bansod2022HTNReplanning,BercherPlanRepairExecute:2014}.

This paper analyzes the formal properties and performance of \ipyhopper, \shopfixer{}, and \rewrite{}.
We compare them formally, and empirically evaluate their performance in a series of benchmark planning problems. 
The following paragraphs summarize our contributions:
\medskip

First, to enable comparison with \rewrite's definition of plan repair \cite{HollerHTNPlanRepair2020}, we formally define the notions of plan repair used by \ipyhopper and \shopfixer. This enables us to prove the following results:
\begin{enumerate}
\item
The definitions correctly characterize the search spaces of
\ipyhopper and \shopfixer.
\item
The sets of solutions for the three algorithms are distinct but not disjoint.
\item
We prove other relationships among the sets of solutions.
\end{enumerate}

Second, we empirically evaluate algorithm performance on three domains drawn from the International Planning Competition (IPC): 
  Openstacks, Rovers, and Satellites. Here are our primary empirical findings:
\begin{enumerate}[start=4]
\item
\rewrite is often the slowest algorithm, because it extensively rederives plans.
\item The causal links and backjumping of \shopfixer are useful for larger problems in improving performance by
  narrowing the size of the search space.
\item The chronological backtracking and simple forward simulation of \ipyhopper can outperform more sophesticated methods
  for smaller problems due to its minimal overhead.
\end{enumerate}

%% file: plan-repair-strategies-new.tex
\section{Preliminaries}
\label{sec:preliminaries}
We begin with the usual classical definitions \cite{GhallabNauTraverso:Book1}, e.g., a state $s$ is a set of ground atoms, an action $a$'s preconditions are $\pre(a)$ \newtext{(in the general case, a First Order Logic formula)}, and if $s$ satisfies $\pre(a)$ then $a$ is applicable in $s$ and $\gamma(s,a)$ is the resulting state. 
A \emph{plan} is a possibly-empty action sequence $\pi =$ $\langle a_1, \ldots, a_n\rangle$.
Applicability of $\pi$ in a state $s_0$,
and the state $\gamma(s_0,\pi)$ produced by applying it, are defined in the usual way. 

To denote the parts of $\pi$ before and after $a_i$, we will write
$\pi[\prec a_i] = \langle a_1,\ldots, a_{i-1}\rangle$
and
$\pi[\succ a_i] = \langle a_{i+1},\ldots,a_n\rangle.$

An \emph{HTN planning domain} is a pair $\Sigma = (\Sigma_c,\M)$, where $\Sigma_c$ is a classical planning domain and $\M$ is a set of methods (see~\cite{GhallabNauTraverso:Book1} for details).
We will only consider \emph{total-order} HTN planning, in which every method's subtasks are totally ordered. 
 We let $M$ be the set of all ground instances of the methods in $\M$.

A \emph{decomposition tree} (d-tree), $T$, represents a recursive decomposition of a task or a task sequence. For a single task $t$, $T$
 has a root node labeled with $t$. For a task sequence $\tau = \langle t_1,\ldots, t_n\rangle$, $T$ has an empty root node whose children are d-trees for the tasks $t_1, \ldots, t_n$.
The primitive tasks (i.e., actions) in $T$ have no children.
Each nonprimitive task $t'$ in $T$ has a child node labeled with a ground method that is relevant for $t'$. 
Each ground method $m$ in $T$ has a sequence of child nodes
labeled with $m$'s subtasks, or if $m$ has no subtasks then it has
one child node labeled by a dummy action with no preconditions and no effects (this simplifies some of the recursive definitions in the next section).

\newtext{We represent $T$'s subtrees with the following notation.}
$T[t]$ is the subtree having $t$ as its root;
$T[\prec t]$ is the subtree containing $t$'s postorder predecessors (i.e., nodes before $t$ in a postorder traversal);
$T[\preceq t]$ is the subtree containing both $T[\prec t]$ and $t$; and
$T[\succ t]$ and $T[\succeq t]$ work similarly.
\smallskip

An \emph{HTN planning problem} is a triple $P = (\Sigma, s_0, \tau_0)$, where $\Sigma$ is the HTN planning domain,
$s_0$ is the initial state, and $\tau_0 = \langle t_1,\ldots,t_n\rangle$ is a (possibly empty) sequence of tasks to accomplish.
A \emph{solution tree} for $P$ is a decomposition tree for $\tau_0$ that is applicable in $s_0$.

Suppose we have executed part of a solution tree $T$ for $P$.
Let $\pi = \plan(T)$ be the sequence of actions at the leaves of $T$, $a_c$ be the last executed action, \newtext{and $s_c$ be the current state} (see Figure \ref{fig:relationships})\newtext{, which will equal $\gamma(s_0, \pi[\preceq a_c])$ unless a disruption occurs.}
 Then the executed and unexecuted parts of $T$
are $T_x= T[\preceq a_c]$ and $T_u= T[\succ a_c]$, and the
executed and unexecuted parts of $\pi$ are
$\pi_x= \pi[\preceq a_c] = \plan(T_x)$
and $\pi_u= \pi[\succ a_c] = \plan(T_u)$.
A task $t$ in $T$ is \emph{fully executed} if $T[t]$ is a subtree of \linebreak
$T[\preceq a_c]$, \emph{unexecuted} if $T[t]$ is a subtree of $T[\succ a_c]$, and \emph{partially executed} otherwise.

\begin{figure}
\centering
\includegraphics[width=.93\columnwidth]{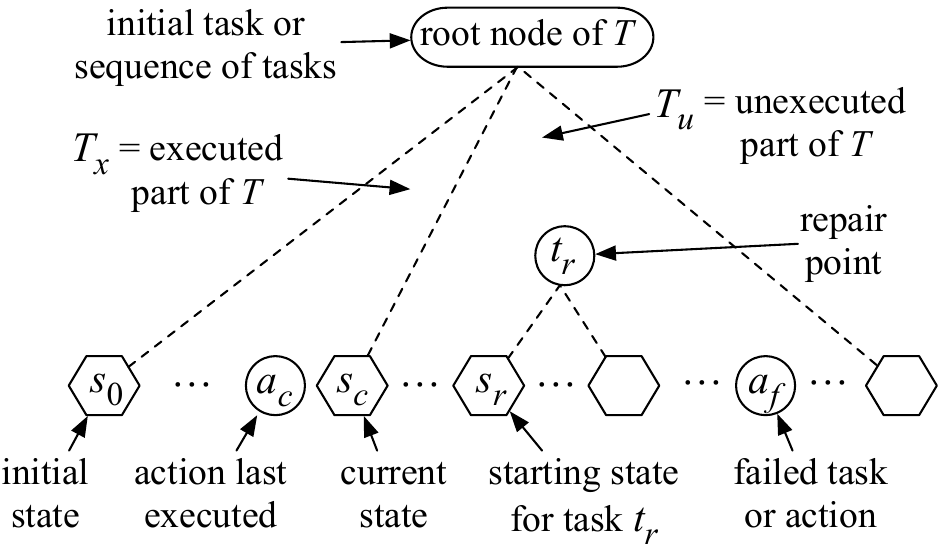}\qquad
\caption{States and tasks relevant for repair of $T$.}
\label{fig:relationships}

\end{figure}

\section{HTN Plan Repair Algorithms}
\label{sec:repair-strategies}

Differences among the three algorithms color the experimental results and are critical to
their interpretation.
Briefly: %
\begin{itemize}[noitemsep,topsep=0pt,parsep=0pt,partopsep=0pt]
\item Because \rewrite (or \rwshort for short) always respects $\pi_x$, it considers unsolvable some of the problems that \ipyhopper (\ipshort) or \shopfixer (\sfshort) solve.
\item \sfshort{} detects potential failures in $T_u$ using causal-link analysis; 
while costly, this means failures can be more rapidly found and repaired.
\item \ipshort{} predicts the future via simulation, which incurs no additional
  cost unless a disturbance actually occurs.

\item \sfshort{} performs backjumping in $T$ while \ipshort{} performs chronological backtracking, 
while again more costly for \sfshort{}, it can resolve repairs in a more holistic manner. 
\item
Section \ref{sec:theory} proves the solution sets are as in
 Figure~\ref{fig:htn_repair_venn_diagram}:
all overlap, but \rwshort{} and \ipshort{} both find some unique solutions and $\{\text{\sfshort 
solutions}\} \subseteq \{\text{\ipshort 
solutions}\}$.
\newtext{\item To repair a plan, \rwshort{} wants an exact match to the original plan's action prefix. So long as they are valid decompositions, there is no restriction on the methods used to produce said prefix. The rest of the plan is only restricted by decomposition validity without regard to plan stability. 
\item \ipshort{} and \sfshort{} try to preserve plan stability by attempting to restrict their changes to as small a part of the plan tree as possible. Unlike \rwshort{}, they will retry failed tasks \textit{ab initio}. For example, suppose task $T$ is initially expanded to $a,b,c$, and after $b$ is executed $c$'s preconditions are unsatisfied.
 If another method is available that expands $T$ to $a',b',c'$, \ipshort{} and \sfshort{} will try this method, giving a repaired plan of $a,b,a',b'c'$. \rwshort{} would reject this plan, since the plan library cannot generate $a,b,a',b'c'$ from $T$.}
\end{itemize}

\begin{figure}
  \centering
  \includegraphics[width=.7\columnwidth]{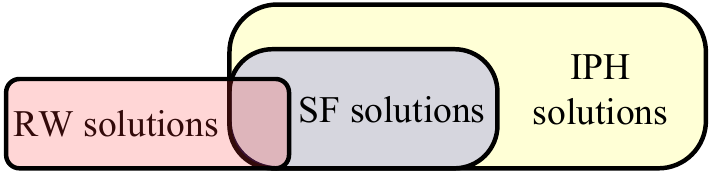}\qquad
  \caption{Venn diagram displaying the relationships among the sets of solutions each algorithm can produce.
   Note that \\
  $\{\text{\rwshort 
solutions}\} \cap \{\text{\ipshort 
solutions}\} \, \subseteq \, \{\text{\sfshort 
solutions}\}$.
  }

  \label{fig:htn_repair_venn_diagram}
\end{figure}

\subsection{\rewrite (\rwshort): Compile a New Problem}
\label{sec:rewrite-alg}

The \emph{rewrite} method of \citet{HollerHTNPlanRepair2020} compiles a new problem and domain definition that is solvable iff the plan can be repaired.  
This compilation forces the planner to build a new plan with a prefix that matches identically $\pi_x$.
As a compilation approach, \rwshort is agnostic to the HTN planner or the search method used.
However, %
there are limits to this flexibility: in practice HTN planners are far
less uniform than PDDL planners.

The new compilation creates a new HTN problem by combining $\Sigma$, $\pi$ (but not $T$), $s_c$, and the disturbance for $a_c$.  
The action $a_c$ is rewritten so that its effects include the disturbance resulting in $s_c$.  
\newtext{A solution to the new problem is any decomposition plan that is consistent with $\pi_x$ and is a complete, grounded decomposition of all the
initial tasks in $T$.}\footnote{Their HDDL~\cite{Hoeller20HDDL} input notation permits problems that have both initial task networks \emph{and} goals.}

\paragraph{\rwshort algorithm implementation}
No runnable implementation of \hetal{}\!'s
algorithm was available,\footnote{Daniel H\"oller, personal
  communication, 26 September 2023.} so we implemented it ourselves; we will share our
implementation on GitHub under an open source license.  
\newtext{\rwshort can use any HDDL-compatible planner, but SHOP3 was the only planner we could find that could parse and solve the experimental problems.\footnote{
We tried Panda, which was only able to parse the smallest domains in 5 minutes (e.g., only Rovers 1-7).  Panda's parsing and grounding algorithms explode in the presence of conditional effects and quantification.  We confirmed this limitation with one of the Panda developers.  Since preparing this manuscript, we have also checked with the IPC2023 HDDL planners Aries \cite{bit2023experimenting} and SIADEX/HPDL \cite{fernandez2021addressing}.  Aries cannot parse our domains, and while SIADEX appears to successfully parse our domains, it is unable to solve them.}}
Our implementation follows the original definition in using HDDL for its
input and output formats.  Our implementation \emph{differs} from the
original definition in being able to handle action and method
\emph{schemas}, rather than only handling ground actions and methods
(\ie, it is a \emph{lifted} implementation).
This required extensions to some parts of the original algorithm.
Our implementation returns a lifted plan repair problem that can be
used directly by the lifted HTN planner \shop{}
\cite{GoldmanKuter:SHOP3ELS}, and that can be grounded for use with
grounded planners.

\subsection{\shopfixer (\sfshort): Examine Causal Links}
\label{sec:shopfixer-alg}

\sfshort \cite{GoldmanETAL:PlanRepair:2020} repairs plans
generated by the forward-searching HTN planner, \shop. 
It constructs a graph of causal links and task decompositions to identify the minimal
subset of $T$ to repair. \sfshort
extends the notion of plan repair stability introduced by
\cite{fox06plan}, and further develops their methods and experiments,
which showed the advantages of plan repair over replanning.

When a disturbance is introduced into the plan at $s_c$, \sfshort 
finds the node $d_f \in T$ whose preconditions are violated; this will either be an action or a method. 
If there is no $d_f$, repair is unnecessary.
Otherwise, \sfshort ``freezes'' $\pi[\preceq a_c]$ and restarts SHOP3 from the parent of $d_f$'s immediate parent. 

Repair proceeds by {\em backjumping} into the search stack and
reconstructing the compromised subtree without the later tasks. 
\sfshort might backjump to decisions prior to $d_f$, but it cannot undo $\pi[\preceq a_c]$.
\sfshort returns $T_x$, $\pi[\preceq a_c]$, and a newly repaired suffix.

Furthermore, if $p$ is the parent of child $c$ in an HTN
plan, then $p$'s preconditions are considered chronologically prior to
$c$'s, because it is the satisfaction of $p$'s preconditions that
enables $c$ to be introduced into the plan: if both $p$ and $c$ fail,
and we repair only $c$, we will still have a failed plan, because
after the disturbance, we are not licensed to insert $c$ or its
successor nodes.

\subsection{\ipyhopper (\ipshort): Repair via  Simulation}
\label{sec:ipyhopper-alg}

\ipshort \cite{Zaidins:HPlan2023} is an augmentation of IPyHOP, a progression-based planner \cite{bansod2022HTNReplanning}.
In contrast to other HTN planners, it does not use projection and instead uses an external simulator to predict action effects.

For plan repair, \ipshort restarts the planning process at \newtext{the task that was decomposed to produce the failed action}
using the current state in place of the stored state.
Initially, \ipshort restricts the process to this subtree and only backtracks further up the tree when all decompositions in the subtree fail.
Once it finds a valid decomposition, \ipshort simulates the action execution going forward. 
If simulation completes, the repair is considered successful. 
If simulation fails at some action $a_f$, \ipshort  restarts a repair process at $a_f$.
This simulation-repair cycle continues until either the plan successfully repaired or root is reached, indicating that
no repaired plan is possible.

%% file: theory.tex
\section{Theoretical Results}
\label{sec:theory}

\subsection{Preliminary Definitions}
\label{sec:theory-preliminaries}

We begin by defining applicability of a decomposition tree $T$. This is more restrictive than just the applicability of $\plan(T)$. Not only must the actions' preconditions be satisfied, but also the preconditions of their ancestor methods. At each node of $T$,
the preconditions that must be satisfied are given by a function $\prestar()$ that is defined recursively:\footnote{Formally, $\prestar$ is a function of a node $n$, but we'll simplify the presentation by writing it as a function of the method or task that labels $n$. We hope the intended meaning is clear.}
\begin{itemize}[noitemsep,topsep=0pt,parsep=0pt,partopsep=0pt]
\item
If a ground method $m$ in $T$ is the child of a task $t$, then
$\prestar(m) = \pre(m) \cup \prestar(t).$
\item
If a nonprimitive task $t$ in $T$ is the first child of a ground method $m$, then
$\prestar(t) = \prestar(m)$. Otherwise
$\prestar(t) = \varnothing$.
\item
If a primitive task (i.e., an action) $a$ in $T$ is the first child of a ground method $m$, then $\prestar(t) = \pre(a) \cup \prestar(m)$.
Otherwise $\prestar(t) = \pre(a)$.
\end{itemize} 
We can now define applicability of $T$ in a state $s$.
Let $\pi = \plan(T)$, and suppose that
for every action $a\in \pi$,
the state
$\gamma(s,\pi[\prec a])$ satisfies $\prestar(a).$
Then $T$ is applicable in $s$, and $\gamma(s,T) = \gamma(s,\pi)$.
Otherwise, $T$ is not applicable in $s$, and $\gamma(s,T)$ is undefined.

Finally, we define a notation for modifications to $T$:
$\replace(T,\langle t_1,\ldots,t_n\rangle,\langle T'_1,\ldots,T'_n\rangle)$ is a modified version of $T$ in which $T[t_i]$ is replaced with $T’_i$ for $i=1,\ldots,n$.
As a special case, $\replace(T,t,T') = \replace(T,\langle t \rangle,\langle T'\rangle)$.

\subsection{Plan Repair Definitions}
 
Let $a_c$ be the last executed action, and $s_c$ be the observed current state. We consider four (non-disjoint) classes of HTN plan-repair problems:

\paragraph{Class 1: normal execution.}
If $s_c = \gamma(s_0,\pi[\preceq a_c])$, then $s_c$ is the predicted state, so no repair is needed.

\paragraph{Class 2: anomaly repair.}
If $s_c \neq  \gamma(s_0,\pi[\preceq a_c])$, then $s_c$ is \emph{anomalous}. To repair this condition, the \emph{repair-by-rewrite} definition involves rewriting the planning domain so
that $a_c$ produces the state $s_c$ instead of the state $\gamma(s_0,T[\prec a_c])$ (for details, see \cite{HollerHTNPlanRepair2020}).
Let $\overline{\Sigma}$ be the rewritten domain, and $\overline{\tau}$ be the rewritten version of $\tau$ in $\overline{\Sigma}$.
Let $\overline{T}$ be any solution tree for the rewritten planning problem $(\overline{\Sigma},s_0,\overline{\tau})$, and $\overline{T}_u$ be the unexecuted part of $\overline{T}$. Let $T'_u$ be the corresponding decomposition tree in $\Sigma$ (i.e., with every rewritten action in $\overline{T}_u$ replaced with its un-rewritten equivalent).
Then $T'_u$ is a repair of $T_u$.

Class 2 is the set of plan-repair problems that are defined in
\cite{HollerHTNPlanRepair2020} and are implemented by the \rwshort algorithm. \rwshort requires replanning whenever an anomaly occurs---but unless there is also a predicted task failure (see Class 3), $T'_u$ may be identical to $T_u$.

\paragraph{Class 3: predicted-task-failure repair.}

In addition to $s_c$ being anomalous, suppose $T_u$ isn't applicable in $s_c$. Then there is a task $t_f$ in $T_u$ such that $\gamma(s_c,T_u[\prec t_f]) \not\models \prestar(t_f)$. This is a \emph{predicted task failure},
which may be repaired as follows.
The \emph{repair point} may be any task $t_r$ in $T_u[\preceq a_f]$.
Let
$s_r$ be the state immediately before $T[t_r].$ Note that in some cases, $s_r$ may precede $s_c.$

If $t_r$ is unexecuted,
let $T_r'$ be any solution tree for the planning problem $(\Sigma,s_r,\langle t_r\rangle)$, and let $T_u' = \replace(T_u,t_r,T_r').$
But if $t_r$ is partially executed, let
$\langle t_1,\ldots,t_n\rangle$ be the postorder sequence of all
unexecuted tasks in $T[t_r]$ that have no unexecuted ancestors, $T_r'$ be any solution tree for the planning problem $(\Sigma,s_c,\langle t_1,\ldots,t_n\rangle)$, and $T_u' = \replace(T_u,\langle t_1,\ldots,t_n\rangle,\langle T_r'[t_1],\ldots,T_r'[t_n]\rangle).$
Then a repair of $T_u$ may defined inductively as follows:
\begin{itemize}[noitemsep,topsep=0pt,parsep=0pt,partopsep=0pt]
\item
(base case) if $T_u'$ is applicable in $s_c$, it is a repair of $T_u.$
    
\item
(induction step) otherwise, $T_u'$ must contain an action $a'_f$
 such that $\gamma(s_c,T_u'[\prec a'_f])$ doesn't satisfy $\prestar(a'_f),$ i.e., $a'_f$ is another predicted failure. In this case, if $T''_u$ is any repair of $T'_u,$ then it is also a repair of $T_u.$
\end{itemize}

\paragraph{Class 4: predicted-action-failure repair.}
In addition to the anomaly and predicted task failure, suppose $\gamma(s_c,\pi_u[\prec a_f]) \not\models {\pre(a_f)}$. This is a \emph{predicted action failure}.
The definition of a repair for this condition is the same as in Class 3, except that only the \emph{actions} of $T_u'$ are considered in the repair.  More specifically:
\begin{itemize}[noitemsep,topsep=0pt,parsep=0pt,partopsep=0pt]
\item The base case applies if $\plan(T_u')$ is applicable in $s_c$,
and otherwise the induction step applies.
\item The induction step uses the preconditions of the failed action (i.e., $\pre(a'_f)$) instead of the preconditions of the tree rooted above the failed action (i.e., $\prestar(a'_f)$).
\end{itemize}

\subsection{Theorems}

The theorems in this section establish properties of \ipshort, \sfshort, and \rwshort, 
assuming that they each are called at the same point in the execution of the same plan.
In particular, the theorems establish relationships among these algorithms and the Class 2, 3, and 4 solutions to HTN plan repair problems.
We will use the following terminology:
\begin{itemize}[noitemsep,topsep=0pt,parsep=0pt,partopsep=0pt]
\item
Since \ipshort and \sfshort
 are nondeterministic, each defines a set of possible solutions to a repair problem. We will call these the
\emph{\ipshort solutions} and \emph{\sfshort solutions}, respectively.
\item
As discussed earlier, Class 2 repair problems and their solutions are essentially the same as in \cite{HollerHTNPlanRepair2020}. We will call the solutions the \emph{\rwshort solutions}.
\end{itemize}

\begin{theorem}
\label{th:class-3-equivalence}
In every HTN plan-repair problem, the set of Class 3 solutions is the set of \sfshort solutions.
\end{theorem}
\begin{proof}
For the base case of Class 3, $T_u'$ is applicable in $s_c$. 
In Algorithm \ref{alg:nondeterministic_shopfixer}, 
\sfshort, such a $T_u'$ will be returned as it meets the criteria for the second branch of the 
If statement. 
The base case of \sfshort mimics the base case for Class 3. 
For the induction step, $T_u'$ is not applicable in $s_c$. 
Therefore there exists at least one action in $T_u'$, $a_f'$ such that 
$\gamma(s_c,\plan(T_u'[\prec a_f'])) \not\models \prestar(a_f')$. 
The third branch of the If statement will handle this and find the first such action, $a_f'$. 
It will then make a recursive call using $T_u'$ and $a_f'$, nondeterministically find an ancestor 
task $t_r$ of $a_f'$, and produce a repaired subtree $T_r$ rooted at $t_r$.
The unexecuted solution tree will then have the repaired subtree inserted as in the replace protocol described. 
This will continue until either plan repair fails or the base case is reached. 
Thus nondeterministic \sfshort mimics the induction step of Class 3. 
As the base case and induction step of Class 3 are mimicked by \sfshort, the set of all possible outputs 
of \sfshort when there is a possible repair is equivalent to Class 3. If no Class 3 repair exists, 
the algorithm will exit at line 5 without returning any solutions, 
hence the set of Class 3 solutions is again the set of solutions returned by \sfshort.
\end{proof}

\begin{theorem}
\label{th:class-4-equivalence}
In every HTN plan-repair problem, the set of Class 4 solutions is the set of \ipshort solutions.
\end{theorem}
\begin{proof} 
Class 4 is the same as Class 3, except that:
\begin{itemize}[noitemsep,topsep=0pt,parsep=0pt,partopsep=0pt]
\item
The base case applies if $\plan(T_u')$ is applicable in $s_c$, and otherwise the induction step applies.
\item
In the induction step, $\prestar(a'_f)$ is replaced with $\pre(a'_f)$.
\end{itemize}
Algorithm \ref{alg:nondeterministic_ipyhopper} is the same as Algorithm \ref{alg:nondeterministic_shopfixer}, except for the above as well.
As Class 3 is to Class 4, so too is \sfshort to \ipshort. As Class 3 is equivalent to the set of solutions producible by \sfshort, Class 4 is equivalent to the set of solutions producible by \sfshort.
\end{proof}

\begin{figure}
\centering
\includegraphics[width=\columnwidth]{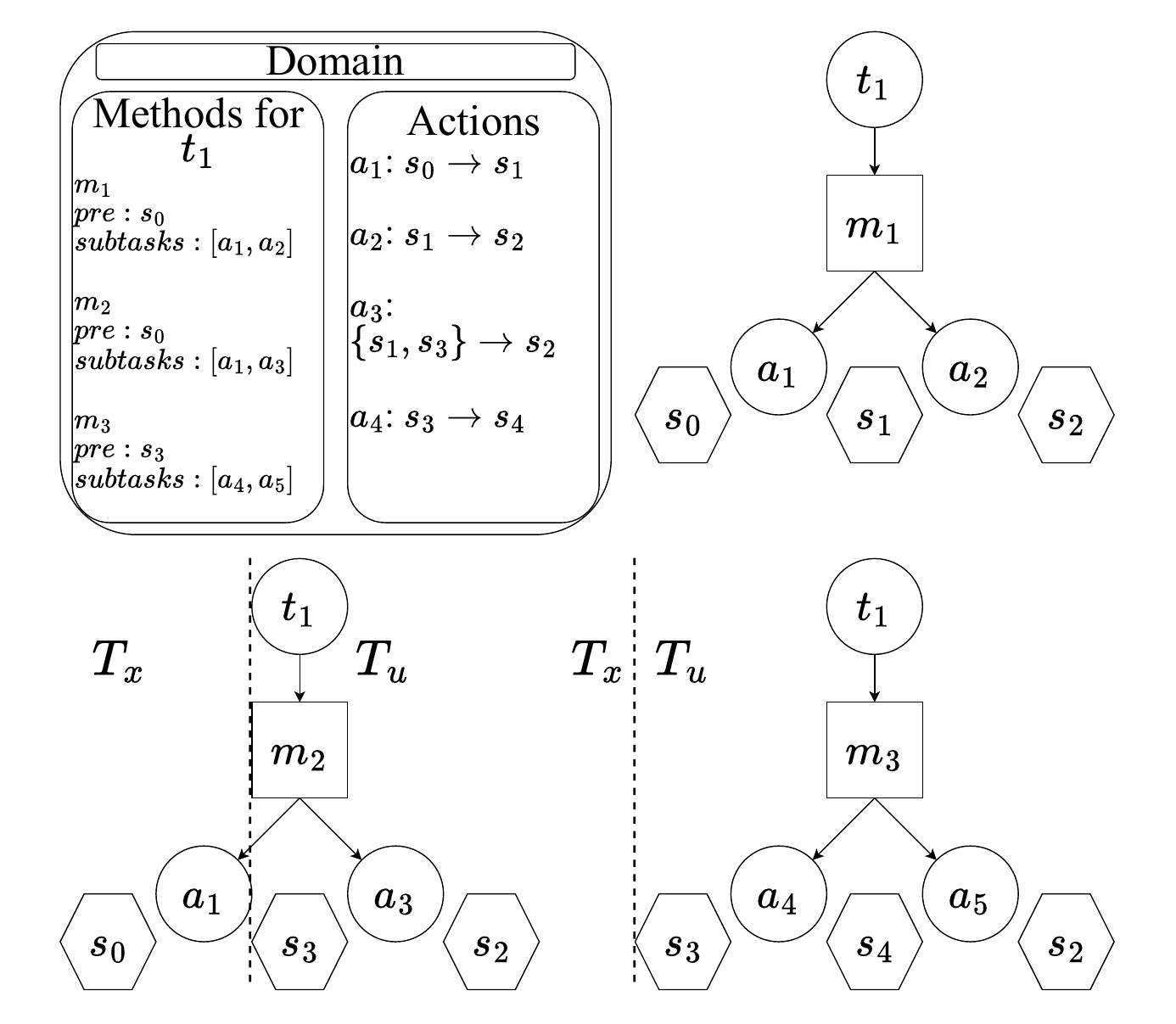}\qquad
\caption{Example demonstrating that there are classes of problems
  that \rwshort can solve but \sfshort and \ipshort cannot, and \textit{vice
  versa}. }
\label{fig:rewrite_vs_shopfixer}
\end{figure}

The remaining theorems in this section establish relationships among the sets of \rwshort, \sfshort, and \ipshort solutions.
The Venn diagram in 
Figure \ref{fig:htn_repair_venn_diagram} illustrates these relationships. 

For Theorems \ref{th:rewrite_not_shopfixer} and \ref{th:shopfixer_not_rewrite}, our proofs use the example shown in Figure \ref{fig:rewrite_vs_shopfixer}. The upper tree is the original solution tree, the lower left
is the tree that is a \rwshort tree and not a \sfshort tree, and the lower right is the \sfshort tree that is not a \rwshort tree. $a_1$ is executed, but an anomalous state, $s_3$, is observed rather than the expected state $s_1$. The state to each action's immediate left is the action's input state, with $s_2$ being the predicted ending state.

\begin{theorem}
There are HTN plan-repair problems that have \rwshort solutions but no \sfshort solutions. 
\label{th:rewrite_not_shopfixer}
\end{theorem}
\begin{proof}
In the bottom left of Figure \ref{fig:rewrite_vs_shopfixer}, we show a \rwshort solution tree that is not a \sfshort solution. The tree is applicable in the initial state, $s_0$, but not in the current state $s_3$. \sfshort would require the parent task of $a_3$, $t_1$, to be decomposed with respect to the observed state, $s_3$. \rwshort allows the methods of executed action be arbitrarily modified so long as the executed action prefix is preserved. \newtext{ With only $m_1$ and $m_2$ in the domain, the problem would have a \rwshort solution, but not a \sfshort solution. Thus there exist HTN plan-repair problems that have RW solutions but no SF solutions.}
\end{proof}

\begin{theorem}
There are HTN plan-repair problems that have \sfshort solutions but no \rwshort solutions.
\label{th:shopfixer_not_rewrite}
\end{theorem}
\begin{proof}
In the bottom right of Figure \ref{fig:rewrite_vs_shopfixer}, we show a solution tree that \sfshort can produce, but not \rwshort. The tree is applicable in the current state, $s_3$, but not in the initial state $s_0$. \rwshort requires the preservation of the executed action prefix in the new solution tree, which is violated here. \newtext{ With only $m_1$ and $m_3$ in the domain, the problem would have a \sfshort solution, but not a \rwshort solution. Thus there exist HTN plan-repair problems that have SF solutions but no RW solutions.}
\end{proof}

\begin{theorem}
\label{th:shopfixer-subset-ipyhopper}
For every HTN plan-repair problem, 
every \sfshort{} solution is also an \ipshort{} solution.
\end{theorem}
\begin{proof}
Recall that the \sfshort solutions and \ipshort solutions are the Class 3 and Class 4 solutions, respectively. 
If $T_u'$ is applicable in $s_c$, then $\plan(T_u')$ is applicable in $s_c$. 
If $\plan(T_u')$ is not applicable in $s_c$, $T_u'$ is not applicable in $s_c$. 
If $s_c \models \prestar(a_f')$, then $s_c \models \pre(a_f')$ as $\prestar(a_f') \models \pre(a_f')$.
Thus for each change we made to Class 3 to describe Class 4 we have loosened the constraints without 
adding constraints. 
Every Class 4 solution meets the criteria of Class 3 solutions, 
but not always the reverse. 
Thus the set of \sfshort solutions is a subset (and often a proper subset) of the \ipshort solutions.
\end{proof}

\begin{corollary}
There are HTN plan-repair problems that have \ipshort solutions but no \rwshort solutions.
\end{corollary}
\begin{proof}
Immediate from the two preceding theorems.
\end{proof}

\begin{theorem}
\label{th:rewrite-ipyhopper-shopfixer}
In every HTN plan-repair problem, if a solution is both a \rwshort solution and \ipshort solution, then it also is a \sfshort solution.
\end{theorem}
\begin{proof}
If $T'_u$ is a \rwshort{} solution, $\overline{T}$ must be applicable in $s_0$. $\overline{T}_u[\succ a_c]$ is the same as $T'_u$, as $\overline{\Sigma}$ is the same as $\Sigma$ beyond the executed action prefix. If $T'_u$ is achievable by \ipshort, $\plan(T'_u)$ is applicable in $s_c$. $T'_u$ is a subtree of $\overline{T}$. Every subtree of an applicable solution tree must be applicable. If $\plan(T'_u)$ is applicable in $s_c$ and $T'_u$ is a subtree of $\overline{T}$, $T'_u$ is applicable in $s_c$. This satisfies the \sfshort criterion of repair and therefore any solution that is both a \rwshort solution and an \ipshort solution must also be a \sfshort solution. %
\end{proof}

\hide{\subsection{Discussion}
\label{sec:theory-discussion}

The theorems establish subset relationships among the sets of solutions that can be found by the nondeterministic \rwshort, 
\sfshort, and \ipshort algorithms. 
}

\begin{algorithm2e}{
\caption{Nondeterministic \ipshort pseudocode}
\label{alg:nondeterministic_ipyhopper}
\small
\SetKwFunction{IPyHOPPER}{\ipshort}
\SetKwFunction{replace}{replace}
\SetKw{return}{return}
\SetKw{continue}{continue}
\SetKw{break}{break}
\SetKw{None}{None}
\SetKwProg{Fn}{Def}{:}{}

\Fn{\ipshort{ current execution state: $s_c$, unexecuted solution tree: $T_u$, failed action: $a_f$ }}{
    $t_r, T_r' \leftarrow$ nondeterministically choose from set of tuples $\{ t_a, T_a$ | $t_a$ is an ancestor of $a_f$, decomposition tree $T_a$ is rooted at $t_a$ and applicable in $ \gamma ( s_c, T_u[ \prec  a_f] \}$ \;
    $T_u' \leftarrow \replace ( T_u, t_r, T_r' )$\;
    \uIf{ $T_r'$ is \None }{
        \return False \;
        
    }
    \uElseIf{ ${\plan(T_u')}$ is applicable in $s_c$ \label{alg:nondeterministic_ipyhopper:T_u} }{ 
        \return $T_u'$ \;
        
    }   
    \uElse{
        $a_f' \leftarrow$ first action, $a \in T_u'$, where ${\gamma ( s_c, T_u'[ \prec  a] )  \not\models \pre(a)}$ \; \label{alg:nondeterministic_ipyhopper:pre}
        \return \ipshort{$s_c, T_u', a_f'$} \;
    }
}
}\end{algorithm2e}

%% file: experiments.tex
\section{Experimental Design}
\label{sec:experimental-design}

\newabbrev{\numBatches}{50}

We tested \sfshort{},
\ipshort{}, and \rwshort{}
 on a
set of identical initial plans and disturbances from three domains:
\textit{Rovers}, \textit{Satellite}, and \textit{Openstacks}.  These
are all HTN domains, formalized equivalently in HDDL and in \shop{}'s
input language.  All of these domains were adapted from International
Planning Competition (IPC)
PDDL domains predating the HTN track,
with HTN methods added and PDDL goals translated into tasks.
These domains (with slightly different disturbances)
were used in a previously published evaluation of
the \sfshort{} plan repair method~\cite{GoldmanETAL:PlanRepair:2020}.

The Satellite and Rover domains each have 20 problems, and the
Openstacks domain has 30.  For each domain, we ran \numBatches{}
batches, where each batch was a run of each problem with one injected
disturbance, randomly chosen and randomly placed in the original plan.
All original plans were generated by \shop{}; they were translated
into HDDL for \ipshort{} and \rwshort{}.  For \ipshort{}, all
inputs were translated into JSON to avoid the need for a new HDDL
parser.  
Runtimes were measured to the nearest hundredth of a
second.  All runtimes were wall-clock times, not CPU times. Runs were forcibly terminated after 300s and in such cases were marked as having failed. \newtext{A runtime is successful when a valid repaired plan is returned.} 

\subsection{Planning Domains}

Here we briefly introduce the domains and deviation operators used in our experiments.  All three domains used the PDDL/HDDL ADL dialect.  Domains were modeled equivalently in both HDDL and
\shop{}; we have indicated below where the \shop{} and HDDL domains
diverged.  All the \emph{original} plans, that is, the plans to repair when execution-time deviations occur,
were generated by \shop{}.

Deviations were modeled similarly to
actions, with preconditions.  Deviation preconditions and effects were defined
in ways that aimed to avoid making repair problems unsolvable, but we were unable to make them completely safe.
Non-trivial plan disturbances were difficult to model without
rendering problems unsolvable because the limited
expressive power of PDDL (and by extension HDDL) forced ramifications
to be ``compiled into'' action effects  (\eg, counting the number of
open stacks in the Openstacks domain), introducing dependencies that
often could not be undone (\eg failing the ``send'' operation in
Openstacks had to also restore the relevant order to ``waiting'').

\paragraph{Rovers}
The Rovers domain is taken from the third IPC in 2002.  Long \& Fox
say it is ``motivated by the 2003 Mars Exploration Rover (MER)
missions and the planned 2009 Mars Science Laboratory (MSL)
mission. The objective is to use a collection of mobile rovers to
traverse between waypoints on the planet, carrying out a variety of
data-collection missions and transmitting data back to a lander. The
problem includes constraints on the visibility of the lander from
various locations and on the ability of individual rovers to traverse
between particular pairs of
waypoints.'' \cite{long03:_inter_plann_compet}
Rovers problems scale in terms of size of the map, number of goals,
and the number of rovers.  Disturbances applied include losing
collected data; decalibration of cameras; and loss of visibility
between points on the map.
For the Rovers problem, the \shop{} domain uses a small set of
path-finding axioms to guide navigation between waypoints.
To avoid infinite loops in the navigation search space, \ipshort{}
does not use lookahead in the waypoint map, but it does check for and
reject cycles in the state space though this check is sound but not complete.

\paragraph{Satellite}
The Satellite problem also premiered in 2002, and is described as
``inspired by the problem of scheduling satellite observations. The
problems involve satellites collecting and storing data using
different instruments to observe a selection of
targets.'' \cite{long03:_inter_plann_compet}  Disturbances used were
changes in direction of satellites, decalibration, and power loss.
Problems scale by number of instruments, satellites and image
acquisition goals.

\paragraph{Openstacks}

IPC 2006 introduced the Openstacks domains as a translation of a standard combinatorial optimization problem,
``minimum maximum open stacks,'' in which
a manufacturer needed to fill a number of orders, each for a combination of different products \cite{GereviniDeterministicplanningfifth2009}.
Problems scale by numbers of products, number of orders, and number of
products in an order.

For plan repair, deviations included removing previously-made products, and
shipping-operation failures. These were difficult to add to Openstacks without
introducing dead ends into the search space because of
consistency constraints on the state that are only implicit in the
operators.
Thus to make repair possible, we added a ``reset'' operation to reset an
order from ``started'' to ``waiting.''

The \shop{} domain for Openstacks included axioms for a cost
heuristic.
Note that this is a common difficulty in plan repair: typically
domains are not written in such a way that recovery is possible if a
plan encounters disturbances because limitations in expressive power
means that ramifications must be programmed into the operator
definitions.  Furthermore, since state constraints (\eg, graph
connectivity in the logistics domain) are not and cannot be captured
in PDDL, one can inadvertently make dramatic changes to problem
structure by introducing disturbances; cf. Hoffmann~\shortcite{hoffmann:jair-11} on problem
topology.

%% file: results.tex
\section{Results}
\label{sec:results}

\paragraph{Satellite}

The Satellite domain was the easiest for all three repair
methods.  Both \ipshort{} and \sfshort{} solved all of
the repair problems in our data set, and found solutions
quickly. \rwshort{} solved the majority of the problems,
between 60\% and 85\% of them
(Figure~\ref{fig:satellite-rewrite-success}).
Inspection shows that it correctly solved all the problems
that did not need repair (\ie, it never timed out re-deriving a plan).
Figure~\ref{fig:satellite-runtimes} gives the runtimes for all three methods,
using $\log_{10}$ because of the wide range of values.
\ipshort{} runtimes are slightly better than \sfshort{} on
average.  The times for \rwshort{} are almost uniformly worse, and scale
worse as the problem size grows. The results for \rwshort{} are not
surprising, since proving unsolvability may take longer.  Indeed,
plotting the runtimes for success and failure separately, demonstrates
that (Figure~\ref{fig:satellite-rewrite}).  Interestingly, there are
no failures due to timeout: \rwshort{} is able to prove unsolvable all
of the unrepairable problems).
While the \ipshort{} times are generally the best on average, its
times vary more widely: see Figure~\ref{fig:satellite-variances}.
Standard deviation of run times is greater for \ipshort{} than
\sfshort{} for all except problem 20.

\paragraph{Rovers}

The Rovers repair problems were more difficult than Satellite, and
none of the repair methods solved all of them.
\ifthenelse{\boolean{supplements}}{%
Figures~\ref{fig:rover-success} and
\ref{fig:rover-success-side-by-side} show}{%
Figure~\ref{fig:rover-success} shows}
success percentages.  Note the
outliers for \ipshort{} and \sfshort{} in problems 3 and 6 (which
are also difficult for \rwshort{}) and for \ipshort{} in
problems 10 and 20.
Generally, \sfshort{} is more successful than \ipshort{}, as shown
in
Table~\ref{tab:rover-repair-pct}\when{supplements}{ and
Figure~\ref{fig:rover-success-side-by-side}}.
\when{supplements}{For full details see Table~\ref{tab:rover-full-success}.}

Generally, while \rwshort{} was less successful than the
other algorithms, it tracks their results except for
problems 14 and 15, where the other two are uniformly successful, but
\rwshort{} is only 86\% and 92\% successful.

Runtimes are graphed in 
\ifthenelse{\boolean{supplements}}{Figures~\ref{fig:rover-runtimes-success},
\ref{fig:rover-runtimes-success-side-by-side},
\ref{fig:rover-runtimes-fail}, and
\ref{fig:rover-runtimes-fail-scatterplot}.}%
{Figures~\ref{fig:rover-runtimes-success}, and 
\ref{fig:rover-runtimes-fail}.}
We plot the successful and failed runs
separately, because the failed runs include both cases where an
algorithm proves that the problem is unsolvable and cases where it
simply runs out of time (time limit was set at 300s). 
\when{supplements}{The scatterplot (Supplementary Figure~\ref{fig:rover-runtimes-fail-scatterplot}) has
the advantage of showing the number of actual failures and
distinguishing visibly between the cases proved unsolvable
and those due to timeouts.}

Again,
\ipshort{} is generally faster, but \sfshort{} scales better with
problem difficulty.  For \ipshort, problem 3 is an
outlier in elapsed time due to failures, \sfshort{} has issues with problem
5, and all three have difficulties with 6, because of failures.
For \rwshort{} the larger problems are quite difficult.
As before, on the successful problems, \ipshort{} has a higher
runtime variance than \sfshort{}: see
Figure~\ref{fig:rover-runtimes-variance} in supplemental for details. 
For all 20 problems the standard
deviation of the runtime is greater for \ipshort{} than \sfshort{}.

\begin{figure}[h]
  \centering
  \includegraphics[width=\columnwidth]{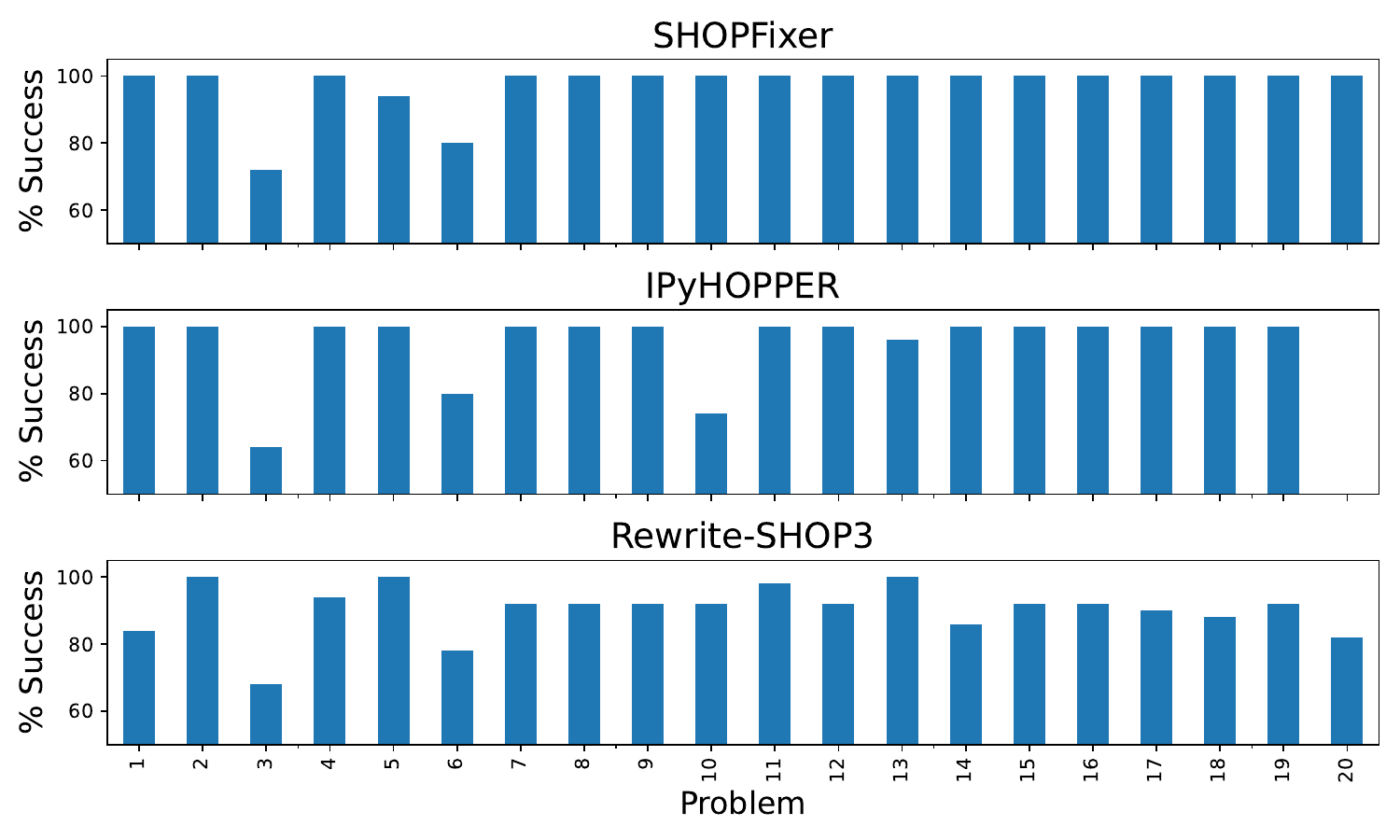}
  \caption{Success rates for the Rovers repair problems for each of
    the three algorithms.}
  \label{fig:rover-success}
\end{figure}

\begin{figure}[h]
  \centering
  \includegraphics[width=\columnwidth]{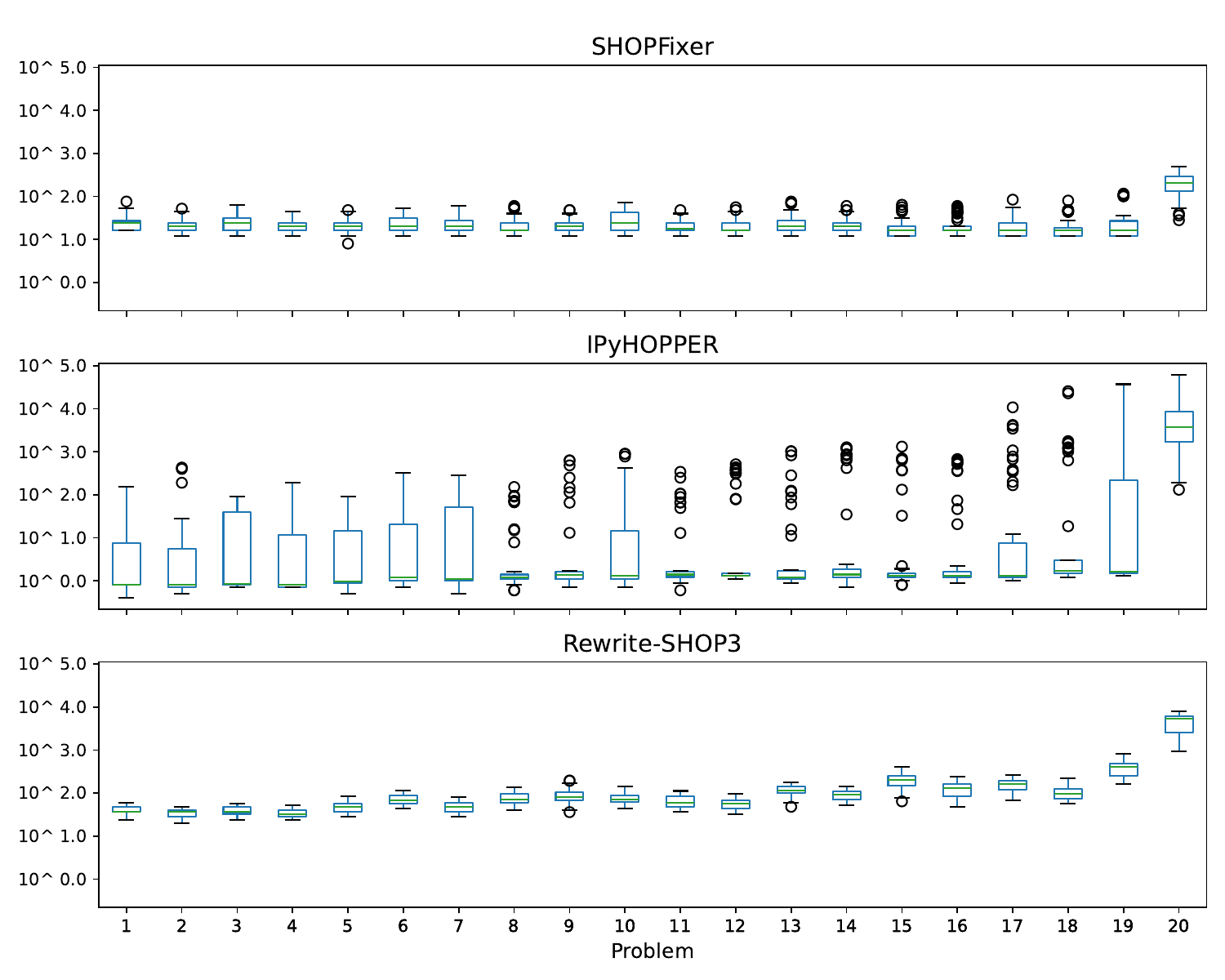}
  \caption{Each algorithm's runtimes in $\mathrm{msec}$ (semi-log plot) on the Rovers repair problems that the algorithm solved
    successfully. 
}
  \label{fig:rover-runtimes-success}
\end{figure}

\begin{figure}[h]
  \centering
  \includegraphics[width=\columnwidth]{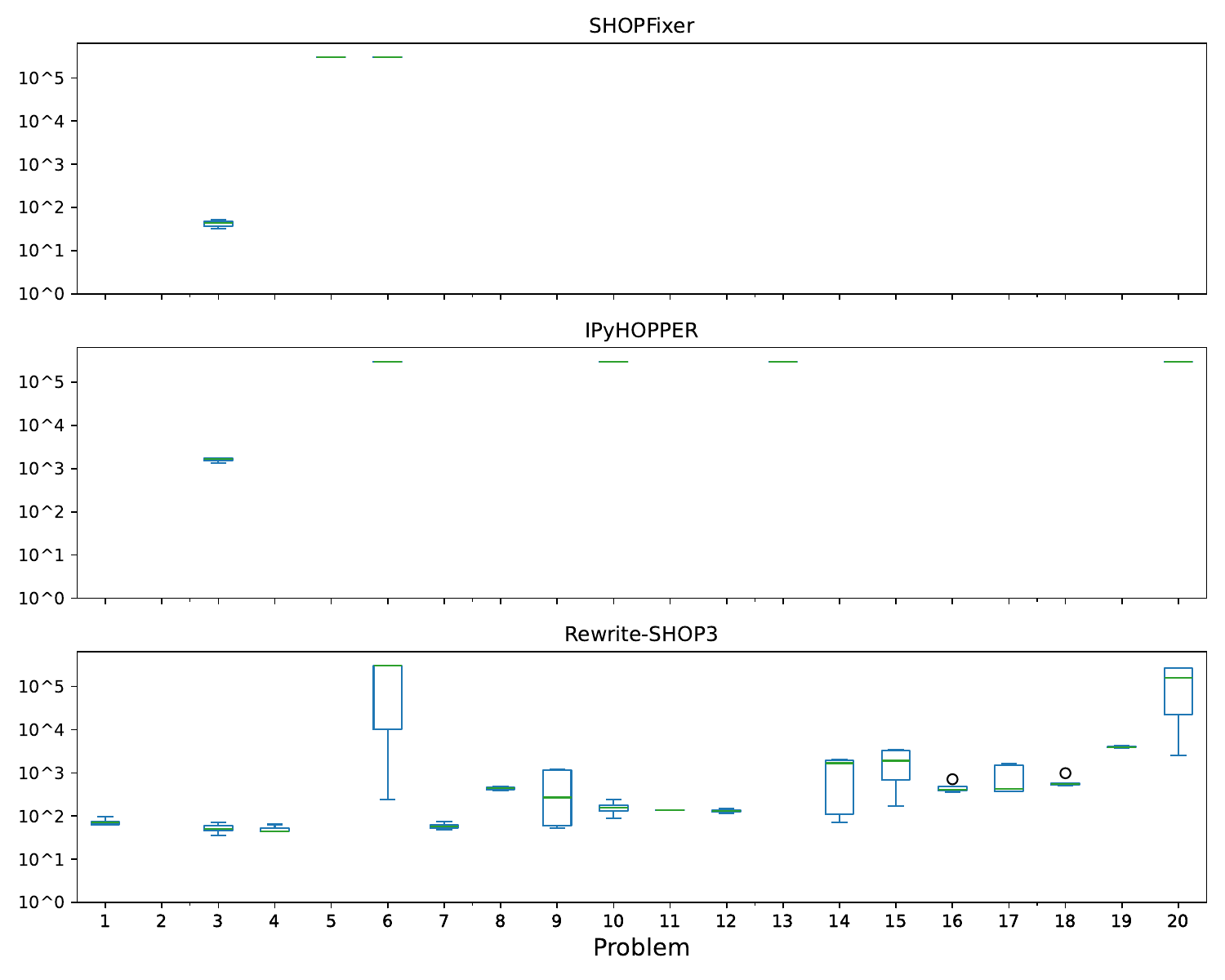}
  \caption{Each algorithm's runtimes in $\mathrm{msec}$ (semi-log plot) on the Rovers repair problems that the algorithm did \emph{not} solve
    successfully.
    }
  \label{fig:rover-runtimes-fail}
\end{figure}

\input{rover-success-table}

\paragraph{Openstacks}

For Openstacks, both \sfshort{} and \rwshort{} solved
all the repair problems.  \ipshort{} solved \emph{almost}
all, but failed for a small number (see
Table~\ref{tab:openstacks-ipyhopper-failures}).  
The runtimes, graphed in
\IfThenElse{supplements}{%
Figures~\ref{fig:openstacks-runtimes} and
Supplementary Figure~\ref{fig:openstacks-runtimes-side-by-side}}{%
Figure~\ref{fig:openstacks-runtimes}} clearly show that \rwshort{} and
\ipshort{} do not scale well on the more difficult problems, with
\rwshort{} notably worse.  \ipshort{} runtimes for problems 3 and 6
are outliers: they are much more difficult than for this solver than
for the others.  Full details are given in Table \ref{tab:openstacks-full-success-table}.

\input{openstacks-success-table}

%% file: rover-success-table.tex
\begin{table}[h]
\centering
  
  \begin{tabular}{c|rr}
  \small
Problem  &  \ipshort \% success &  \sfshort \% success\\ \hline 
~~3 & 64\%\quad{} & \textbf{72\%\quad{}} \\
~~5 & \textbf{100\%\quad{}} & 94\%\quad{} \\
~~6 & \textbf{80\%\quad{}} & \textbf{80\%\quad{}} \\
10 & 74\%\quad{} & \textbf{100\%\quad{}} \\
13 & 96\%\quad{} & \textbf{100\%\quad{}} \\
20 & 36\%\quad{} & \textbf{100\%\quad{}} \\
\end{tabular}

\caption{Success rates for Rover problems where either \ipshort or \sfshort{} did not
    solve all repair problems.  Problems not listed were all solved by both
    repair methods.
}  \label{tab:rover-repair-pct}
\end{table}

%% file: openstacks-success-table.tex
\begin{table}[h]
\centering
\begin{tabular}{c|c}
\small
Problem  &  \% success\\ \hline
~~7 & 98\% \\
10 & 98\% \\
14 & 98\% \\
19 & 98\% \\
25 & 94\% \\
\end{tabular}\quad
\begin{tabular}{c|c}
Problem  &  \% success\\ \hline
27 & 96\% \\
28 & 72\% \\
29 & 78\% \\
30 & 78\% \\
\mbox{}\\
\end{tabular}

\caption{Success rates for Openstacks problems where \ipyhopper{} did not solve all of
    the problems.}
  \label{tab:openstacks-ipyhopper-failures}
\end{table}

%% file: discussion.tex
\section{Discussion}
\label{sec:discussion}

\paragraph{Common Features}

Across all of the domains, \rwshort{} is less time-efficient than the
other two repair methods.  This is due to the fact that it replans
\emph{ab initio}, albeit against a new problem that forces the plan to
replicate already-executed actions.  This involves an extensive amount
of rework, as can be seen very clearly in the Openstacks problems,
which have the highest runtimes for generating initial
plans\when{supplements}{ (see Supplementary Figure~\ref{fig:orig-runtimes})}.
This could likely be substantially improved
by heuristic guidance to direct the early part of planning
towards methods that replicate actions previously seen and that avoid
infeasibly introducing new actions.

\paragraph{Satellite}
It is unsurprising that \rwshort{} cannot solve some
problems that the other two algorithms solve.  Many of the
repairs for satellite simply involve immediately restoring a condition deleted by
a disturbance---and if it was a condition that the plan
established, a re-establishment typically will not work.  Perhaps the surprise
is that \emph{any} of these scenarios are successfully repaired by
\rwshort{}.  We investigated further, and found that \rwshort{} could handle all of the cases that did not require repair
(where the disturbance did not defeat any action preconditions).
Removing those cases gives the success rates shown in
Figure~\ref{fig:satellite-rewrite-success}.

Here is an example of how \rwshort's definition of repair makes it unable to solve a problem
handled by the other two systems.  In this repair problem
\when{supplements}{ (see
Table~\ref{tab:satellite-unsolvable})}, \verb|instrument0| becomes
decalibrated after it has been calibrated and pointed at its
observation target (\verb|phenomenon6|).  The way the domain is
written, calibration occurs only in a sequence of
calibration then observation. Thus there is no
plan in which two calibration operations are not separated
by an observation, so repair by rewrite is impossible.  The other two
methods simply treat the preparation task as having failed, and
re-execute the calibration, because their repair definition is more
permissive.

For this domain, \ipshort{} is generally faster than \sfshort{},
although it is implemented in interpreted Python rather than compiled
Common Lisp. This is probably accounted for by the fact that
\sfshort{} invests in building complex plan trees that include
dependency information in order to more rapidly identify the location
For these simple
problems, \sfshort{}'s added effort is generally not worthwhile.  We note that the
\emph{variance} of \ipshort{}'s runtimes is wider than that of
\sfshort{}, and that there are more outliers (Figure~\ref{fig:satellite-variances}).

\paragraph{Rovers}
There were several Rovers problems where even \ipshort{}
and \sfshort{} could not find solutions---but the three
algorithms behaved quite differently in these cases.  There were 99
Rovers problems that \rsthree{} could not repair.  Of these, only 2
were due to timeouts, both on problem 6, showing that
the algorithm usually could prove problems unrepairable.
As before, \sfshort's more permissive definition of ``repairable'' meant that
it solved more problems: it found only 35 unrepairable,
and of these only 3 were due to timeouts, as with \rsthree{} these
were both for problem 6.
\ipshort{} had much more difficulty with these problems.  It
failed to repair 97 cases, of which 78 were due to timeouts.
Timeouts are not a simple matter of scale: the greatest number of
timeouts (by a factor of 2) is for problem 20, but the runners-up are
3, 6, and 10, in declining order of number of timeouts.  Since the
problems are intended to scale from first to last, the
outcomes are not due only to raw scale.

Repair difficulties in the Rovers domain are due to the nature of the
disturbances in our model.  The ``obstruct-visibility'' disturbance
can render the waypoint graph no longer fully connected, in terms of
rover reachability.  Losing a sample may also give rise to an
unrepairable problem.

\paragraph{Openstacks}
The hardest of the domains, Openstacks shows the benefit of
\sfshort{}'s more expensive tree representation in runtime.  We can
see this even more clearly if we plot the two algorithms directly
against each other (Figure~\ref{fig:openstacks-comparison}).  Indeed,
\ipshort{}'s failures in this domain are all due to timeouts.
Specific problems are indicated in
Table~\ref{tab:openstacks-ipyhopper-failures}.  The more difficult
search space here heavily penalizes \ipshort{}'s simple
chronological backtracking.

In a reversal of the previous patterns, \rwshort{} solves all of the
problems.  This is due to a difference in the way the domain was
formalized compared to the other two domains. Recall that we had to
add a new action to prevent disturbances from making the Openstacks
problems unsolvable.  That modification had the effect of also helping
\rwshort{} as did the fact that action choice is primarily constrained
by preconditions, rather than by method structure, which also avoided
issues with this algorithm.  Note that this relatively unconstrained
planning also made it more difficult to generate the initial plans for
this domain.

%% file: conclusions.tex
\section{Conclusions and Future Directions}
\label{sec:conclusions}

We have presented an analysis of three recent hierarchical repair algorithms from the AI planning literature; namely, \sfshort, \ipshort, and \rwshort. Qualitatively, our analyses highlighted significant differences among these methods:
First, \rwshort's deﬁnition of plan repair is more stringent than the others in practice;
we have identiﬁed benchmark planning problems that are solvable for \ipshort\ and/or \sfshort\ that cannot be solved by \rwshort. Secondly, \sfshort\ attempts to detect when a plan will be invalid, before any actions actually fail; i.e., \sfshort\ invests in both data structures
and computation in order to detect compromise to a plan as
soon as possible. \ipshort, on the other hand, is not a
model-based projective planner in the same sense: it relies
on an external simulation to do projection for it, instead of
having an internal action model as most planners do. This
difference in planning approaches leads to different plan repair behaviors.

Our results on the efficiency of the \rwshort{} algorithm should be
taken with a large grain of salt. The original developers of this
algorithm point out that their characterization is intended to be
conceptually correct and clean, and that they have not yet taken into
account the efficiency of the formulation. In addition to tuning the
formulation, its efficiency could be improved by improved heuristics
for planner when they run against rewrite problems. In particular, a
planner searching the decomposition tree top-down should take into
account the position of its leftmost child when deciding whether to
choose the original methods, or methods whose leaves are taken from
the executed prefix of the plan.

Our experiences also highlight unresolved issues in applying
\rwshort{} in grounded planning systems.  How best to schedule re-grounding \textit{vis-\`{a}-vis} generation of the rewritten repair
problem remains to be determined. While there were some subtleties
to resolve in developing our lifted implementation, it did not have this chicken-and-egg problem.

Another interesting research direction is studying how HTN domain
engineering  affects the tradeoffs between
efficiency and flexibility. 
At present, repair problems are generally created by modifying
previously-existing planning problems (including IPC problems), which
are not designed for execution, let alone to be repairable.
In connection with the general concept of
stability, this may yield a new insights in search-control for plan
repair and for repairable plans; deriving properties from how
preconditions and effects enable planning heuristics and repairability as
well as how those preconditions and the task structure enable search
control at higher levels of the plan trees. Similar to {\em
  refineability} properties
\cite{bacchus1992expected,yang1997generating}, this approach can be
examined formally, theoretically, and experimentally.

%% file: supplemental.tex
\appendix
\onecolumn{}
\renewcommand\thefigure{S\arabic{figure}}    
\renewcommand\thetable{S\arabic{table}} 
\setcounter{figure}{0} 
\setcounter{table}{0} 

\makeatletter
\def\@seccntformat#1{%
  \expandafter\ifx\csname c@#1\endcsname\c@section\else
  \csname the#1\endcsname\quad
  \fi}
\makeatother

\section{Supplementary Material}

\begin{figure}[h]
  \centering
  \includegraphics[width=0.9\columnwidth]{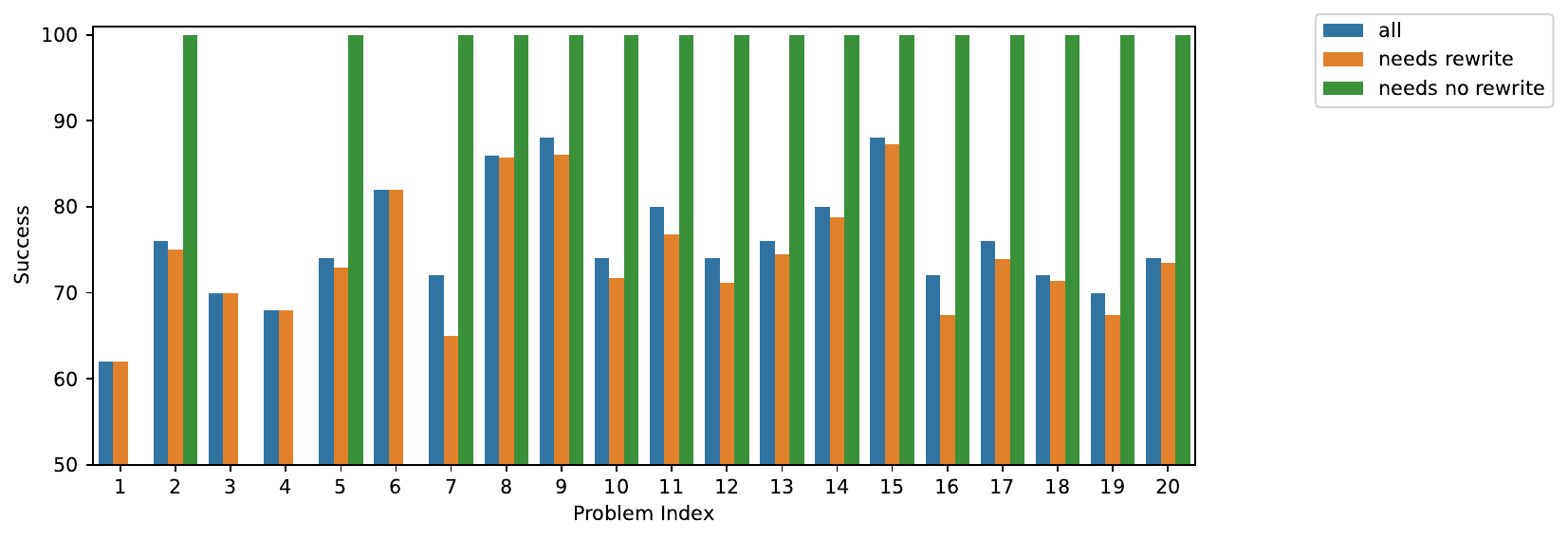}
  \caption{\rsthree{} success rates for the satellite repair
    problems (note the $y$ axis runs only from 50-100\%).}
  \label{fig:satellite-rewrite-success}
\end{figure}

\begin{figure}[h]
  \centering
  \includegraphics[width=0.9\columnwidth]{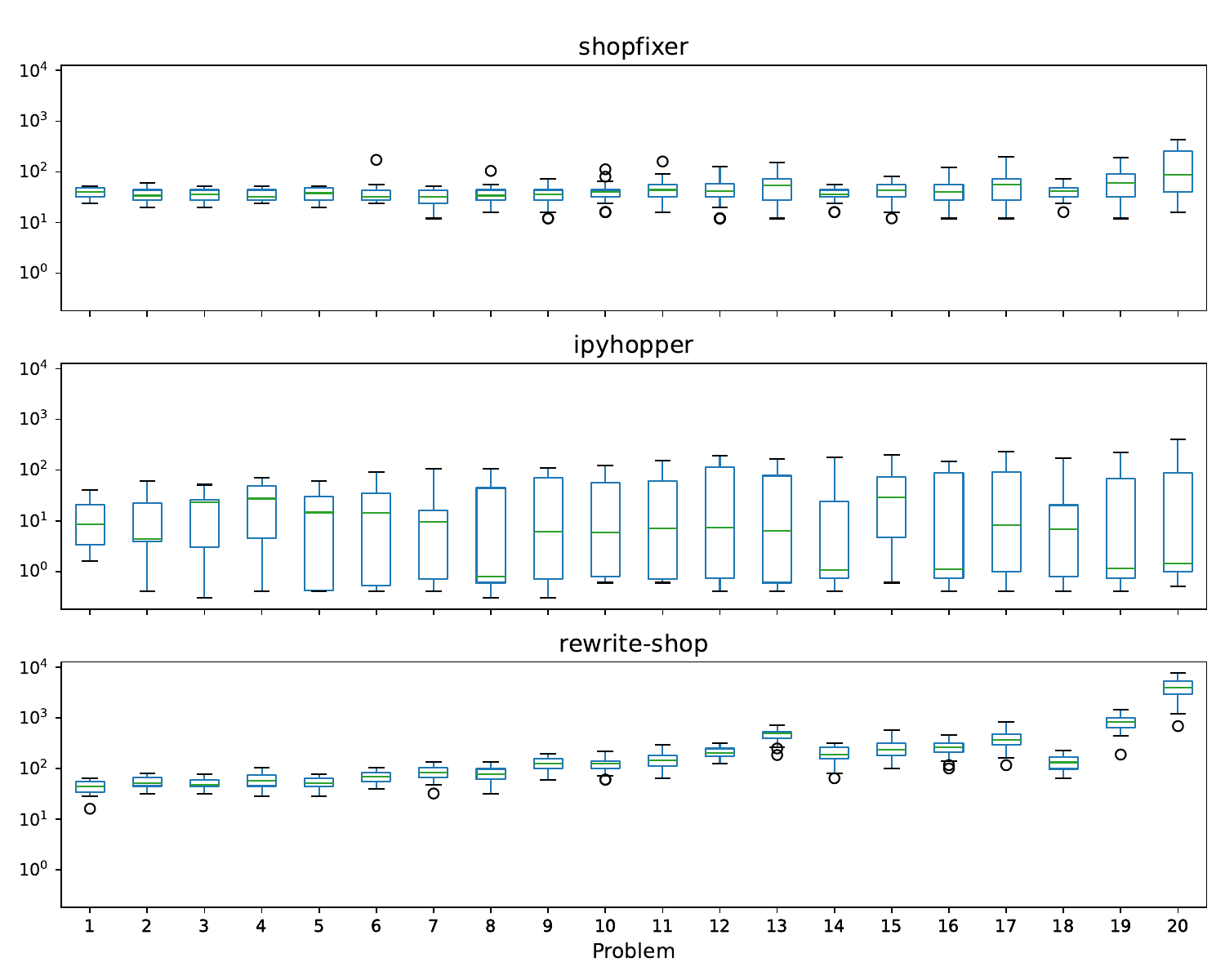}
  \caption{Satellite problem runtimes for all repair algorithms in $\mathrm{msec}$ (semi-log plot).}
  \label{fig:satellite-runtimes}
\end{figure}
\begin{figure}[h]
  \centering
  \includegraphics[width=0.9\columnwidth]{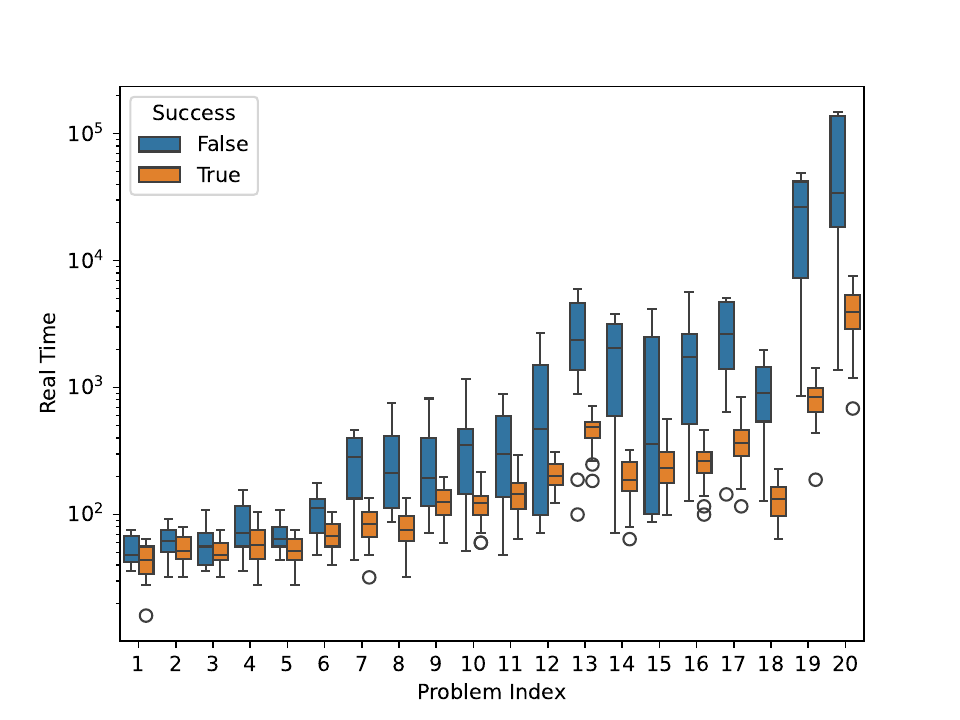}
  \caption{Satellite problem runtimes for \rsthree{}, comparing
    successful trials versus failed trials, plotted in $\log_{10}(\mathrm{msec})$.}
  \label{fig:satellite-rewrite}
\end{figure}

\begin{figure}[h]
  \centering
  \includegraphics[width=0.9\columnwidth]{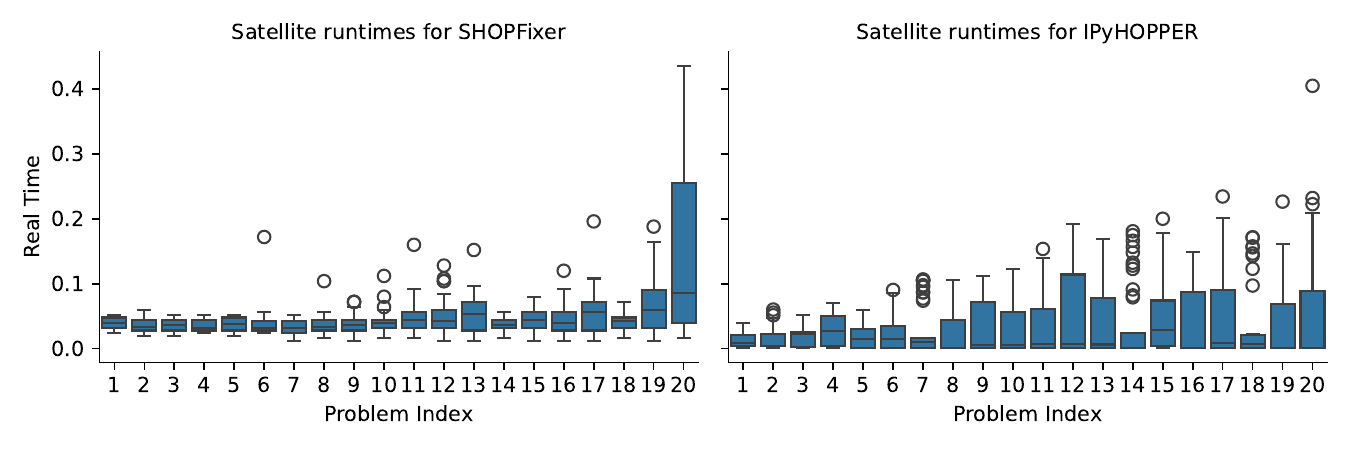}
  \caption{Satellite problem runtimes vary more greatly for
    \ipyhopper{} than \shopfixer{}. Problem 20 is a notable exception.}
  \label{fig:satellite-variances}
\end{figure}

\begin{figure}[h]
  \centering
  \includegraphics[width=0.9\columnwidth]{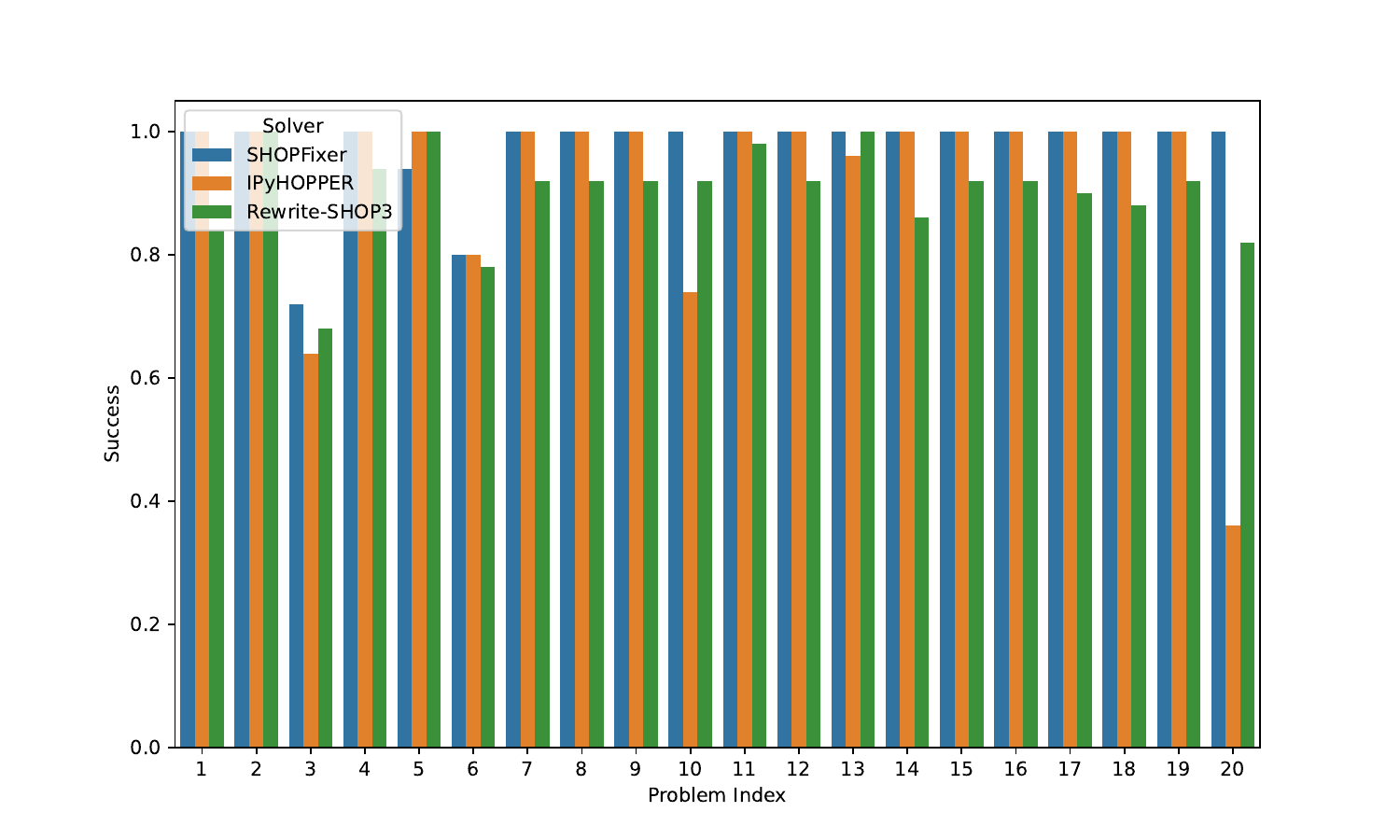}
  \caption{Success rates for the Rovers repair problems for each of
    the three algorithms, plotted side-by-side.}
  \label{fig:rover-success-side-by-side}
\end{figure}

\begin{figure}[h]
  \centering
  \includegraphics[width=0.9\columnwidth]{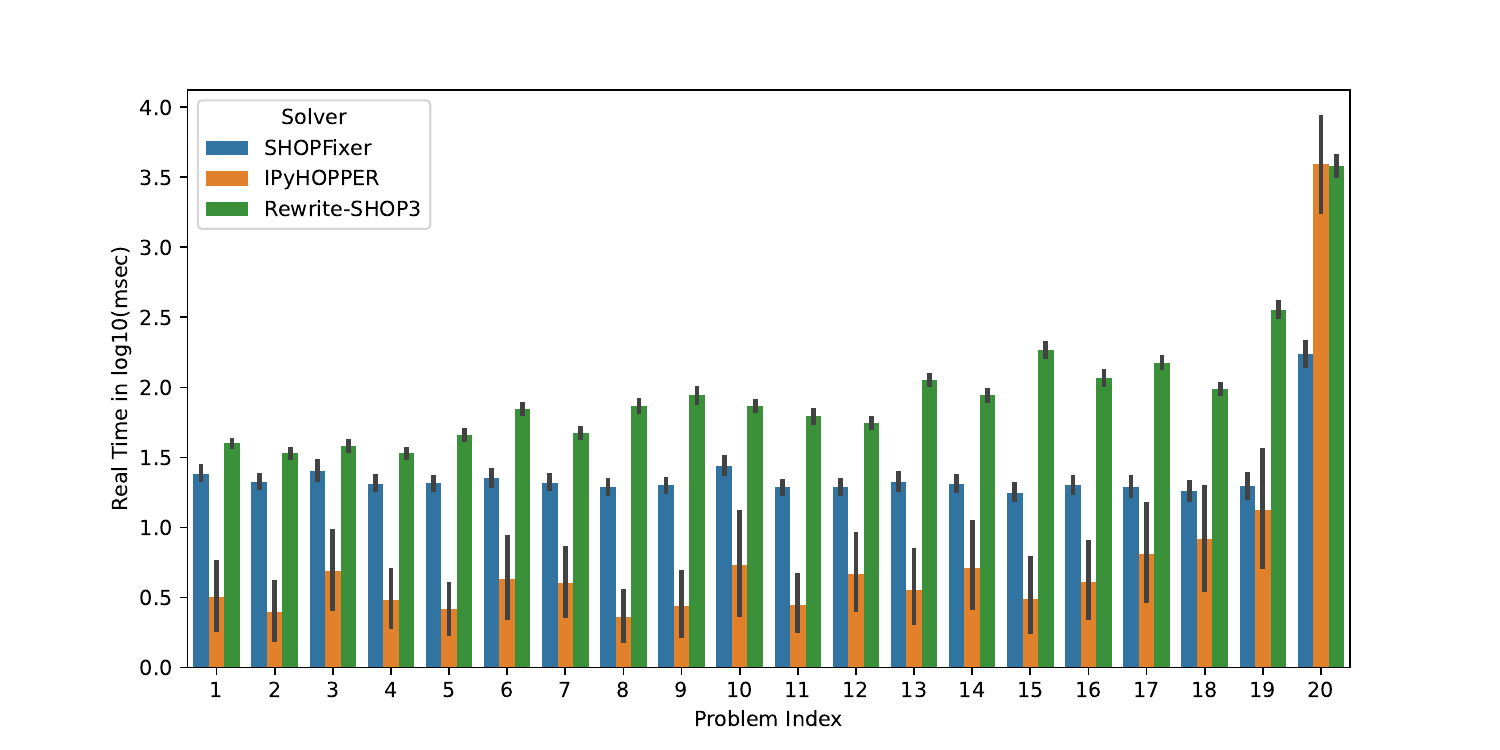}
  \caption{Runtimes for the Rovers repair problems in
    $\log_{10}(\mathrm{msec})$,
    plotted side-by-side for easier comparison.  These include only those problems solved
    successfully by the algorithm in question.}
  \label{fig:rover-runtimes-success-side-by-side}
\end{figure}

\begin{figure}[h]
  \centering
  \includegraphics[width=0.9\columnwidth]{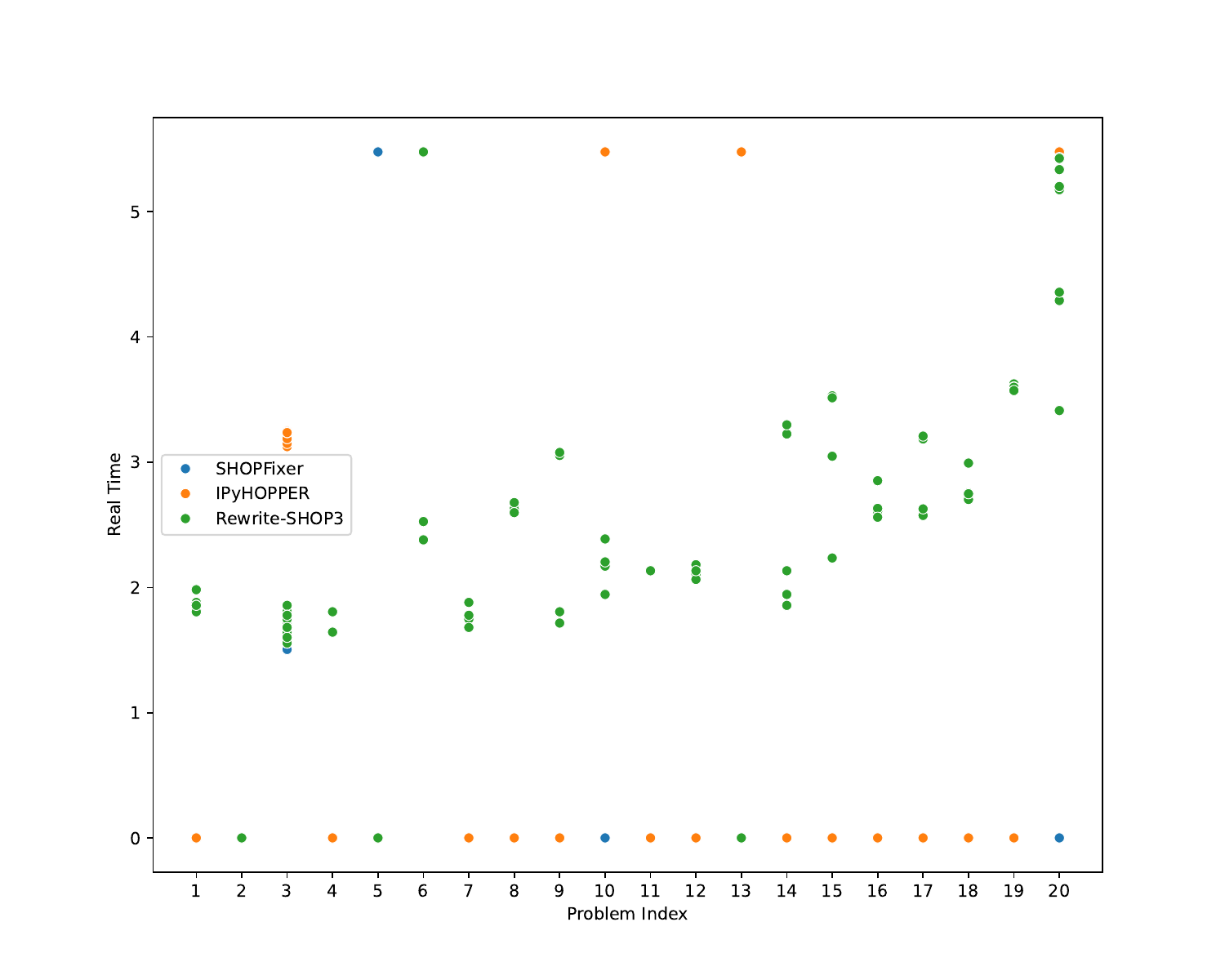}
  \caption{Runtimes for the Rovers repair problems in $\log_{10}(\mathrm{msec})$ for each algorithm.  These are only the problems \emph{not} solved
    successfully by the algorithm in question.}
  \label{fig:rover-runtimes-fail-scatterplot}
\end{figure}

\begin{figure}[h]
  \centering
  \includegraphics[width=0.9\columnwidth]{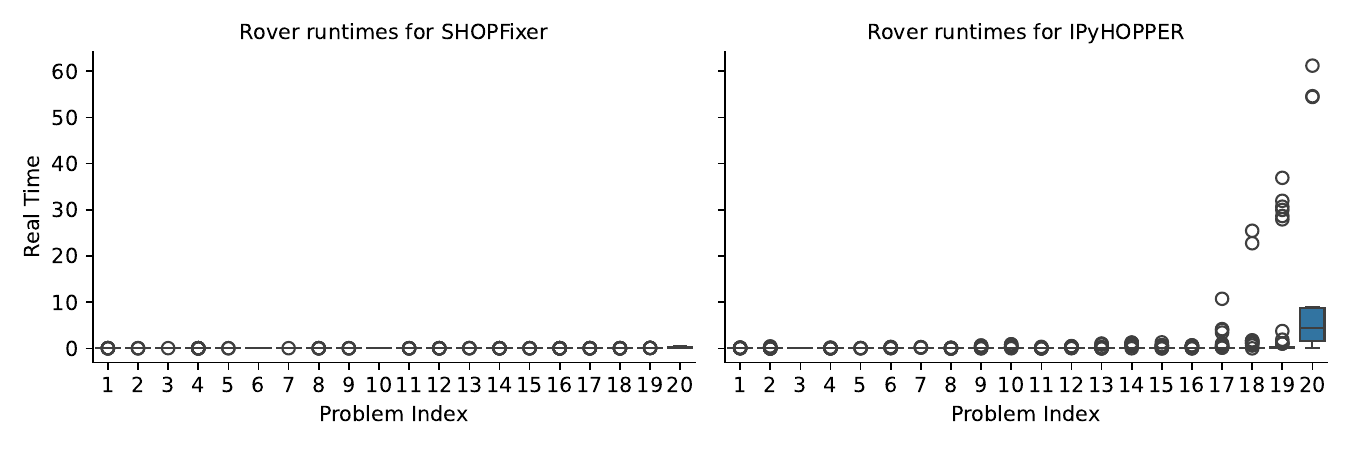}
  \caption{Variance in runtimes for the Rovers repair problems for
    \ipshort{} and \sfshort{}.   Only successfully solved problems
    are plotted.}
  \label{fig:rover-runtimes-variance}
\end{figure}

\input{rover-full-success-table}

\begin{figure}[h]
  \centering
  \includegraphics[width=0.9\columnwidth]{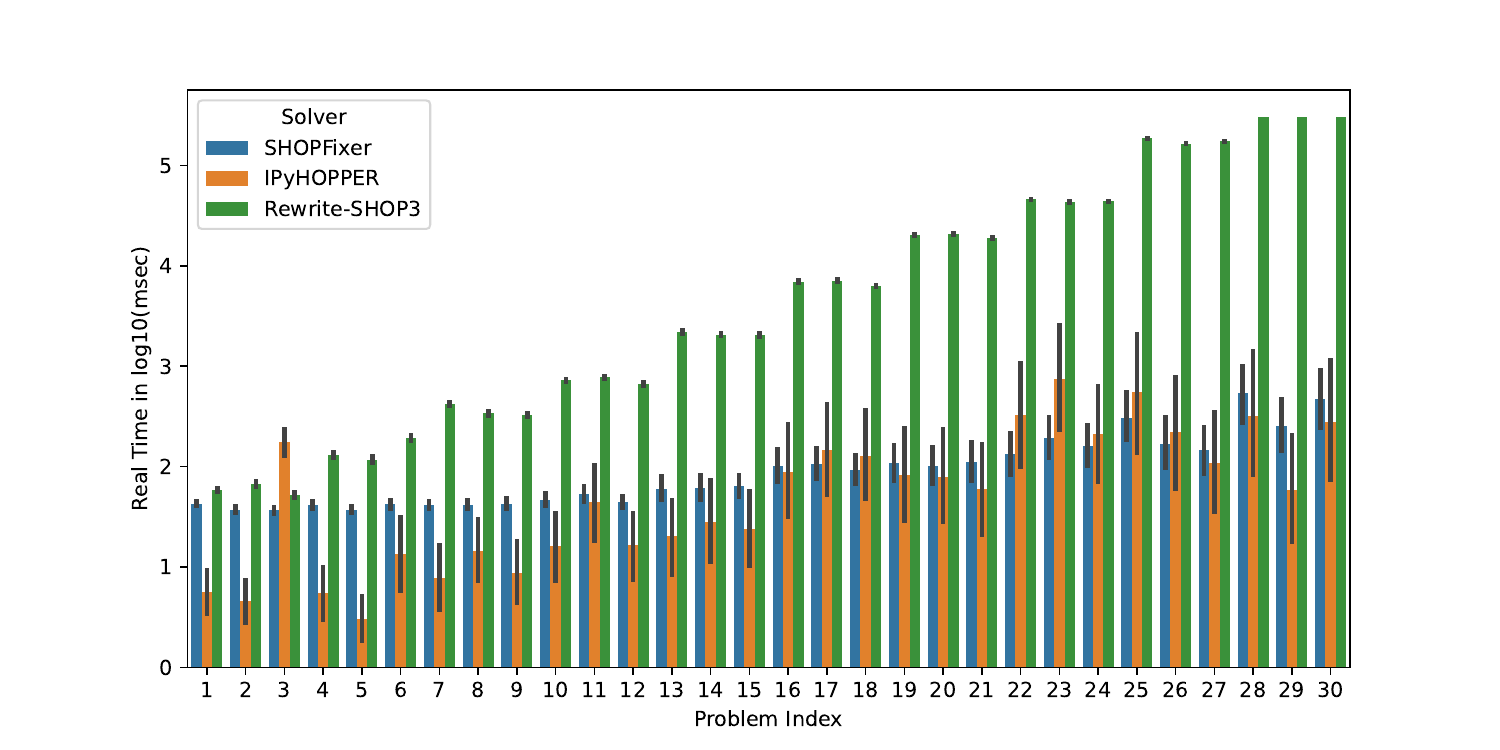}
  \caption{Runtimes for successfully solved Openstacks repair problems in $\log_{10}(\mathrm{msec})$ for each of
    the three algorithms, plotted side-by-side for easier comparison.}
  \label{fig:openstacks-runtimes-side-by-side}
\end{figure}

\input{openstacks-full-success-table}

\input{satellite-fail-case}

\begin{figure*}[h]
  \begin{tabular}{cc}

\hspace*{-1.7em}\includegraphics[width=0.56\textwidth]{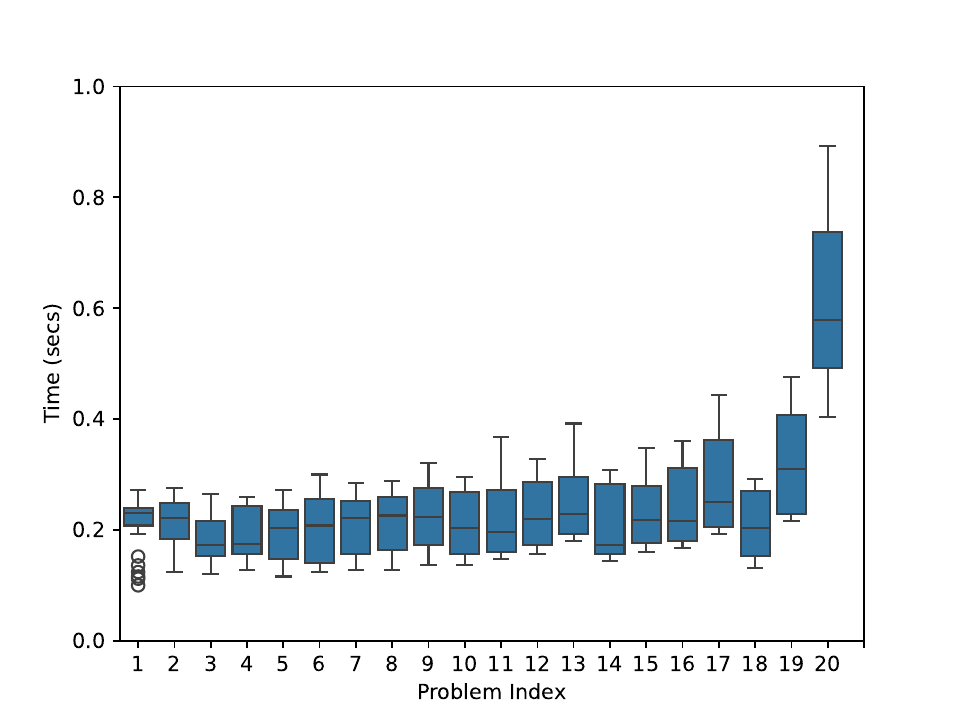}\hspace*{-1.7em}
    &
\hspace*{-1.7em}\includegraphics[width=0.56\textwidth]{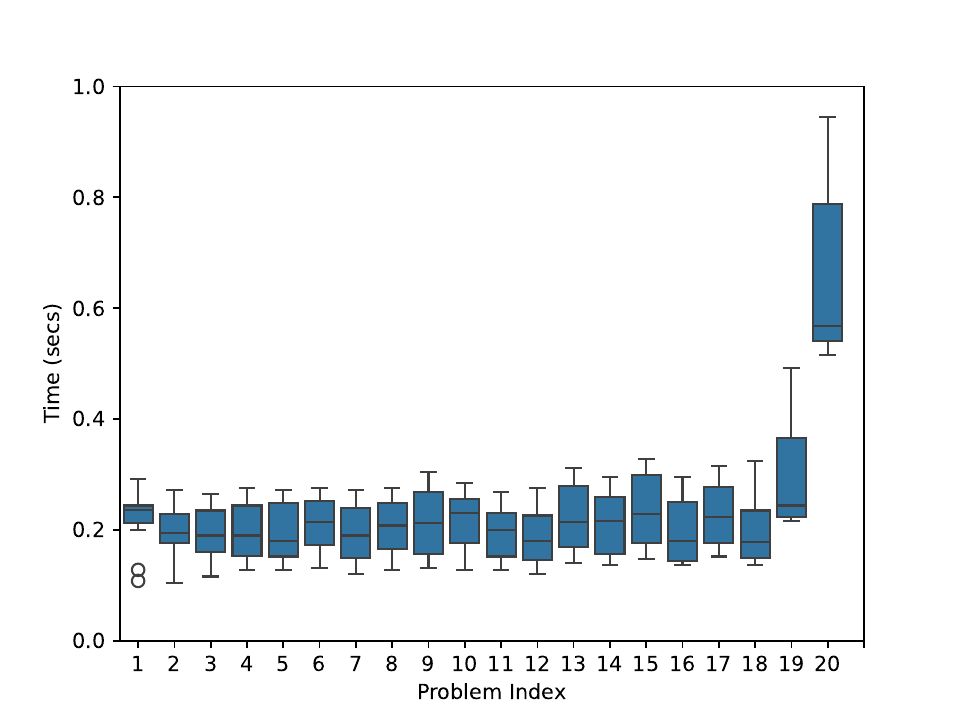}\hspace*{-1.7em}
  \\
Satellite & Rovers \\
    \multicolumn{2}{c}{\includegraphics[width=0.7\textwidth]{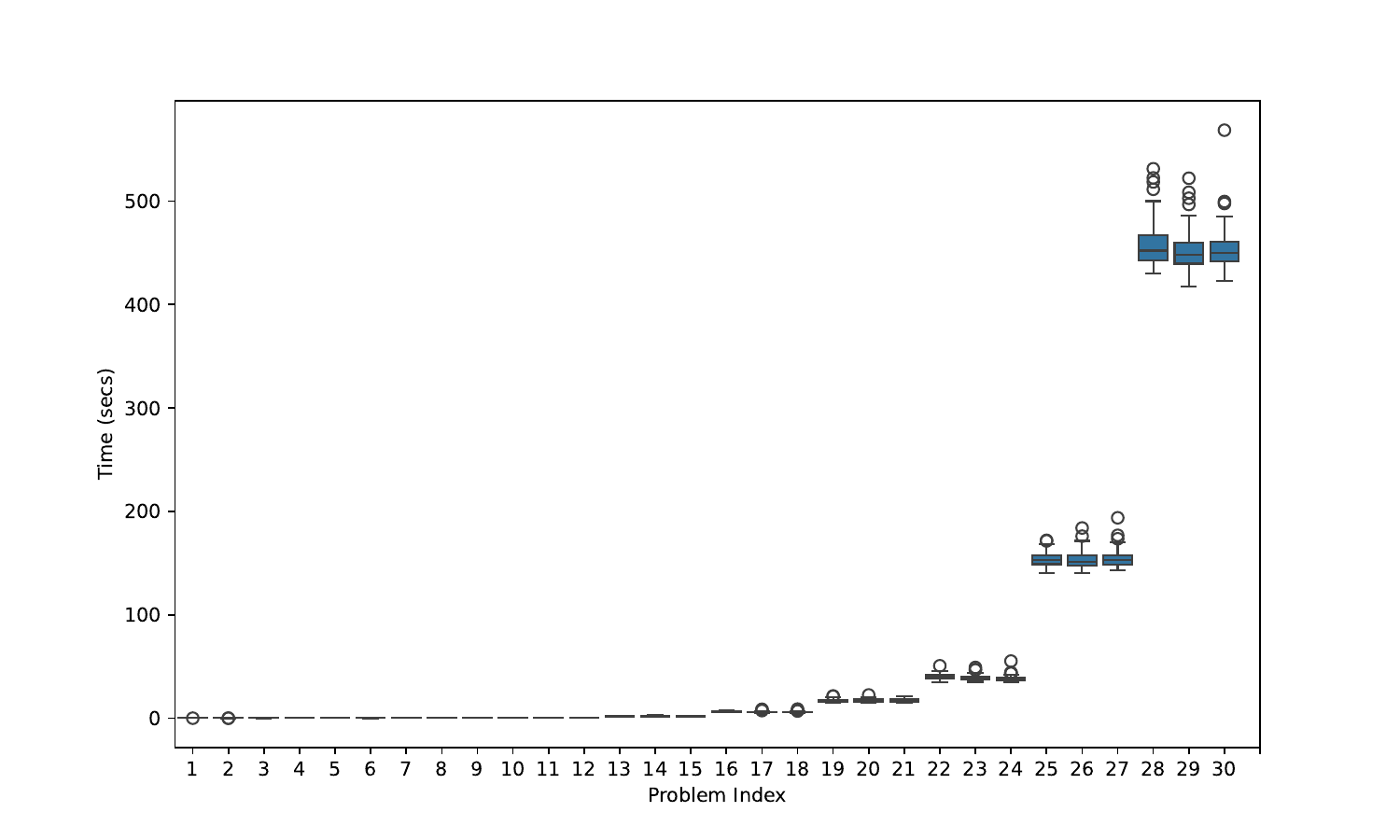}} \\
    \multicolumn{2}{c}{Openstacks}
  \end{tabular}
  \caption{Runtimes for generating the original (to be repaired) plans
    for the three domains, all using \shop{}. Note the difference in
    $y$ axis for Openstacks.}
  \label{fig:orig-runtimes}
\end{figure*}

\begin{figure}[h]
  \centering
  \includegraphics[width=0.9\columnwidth]{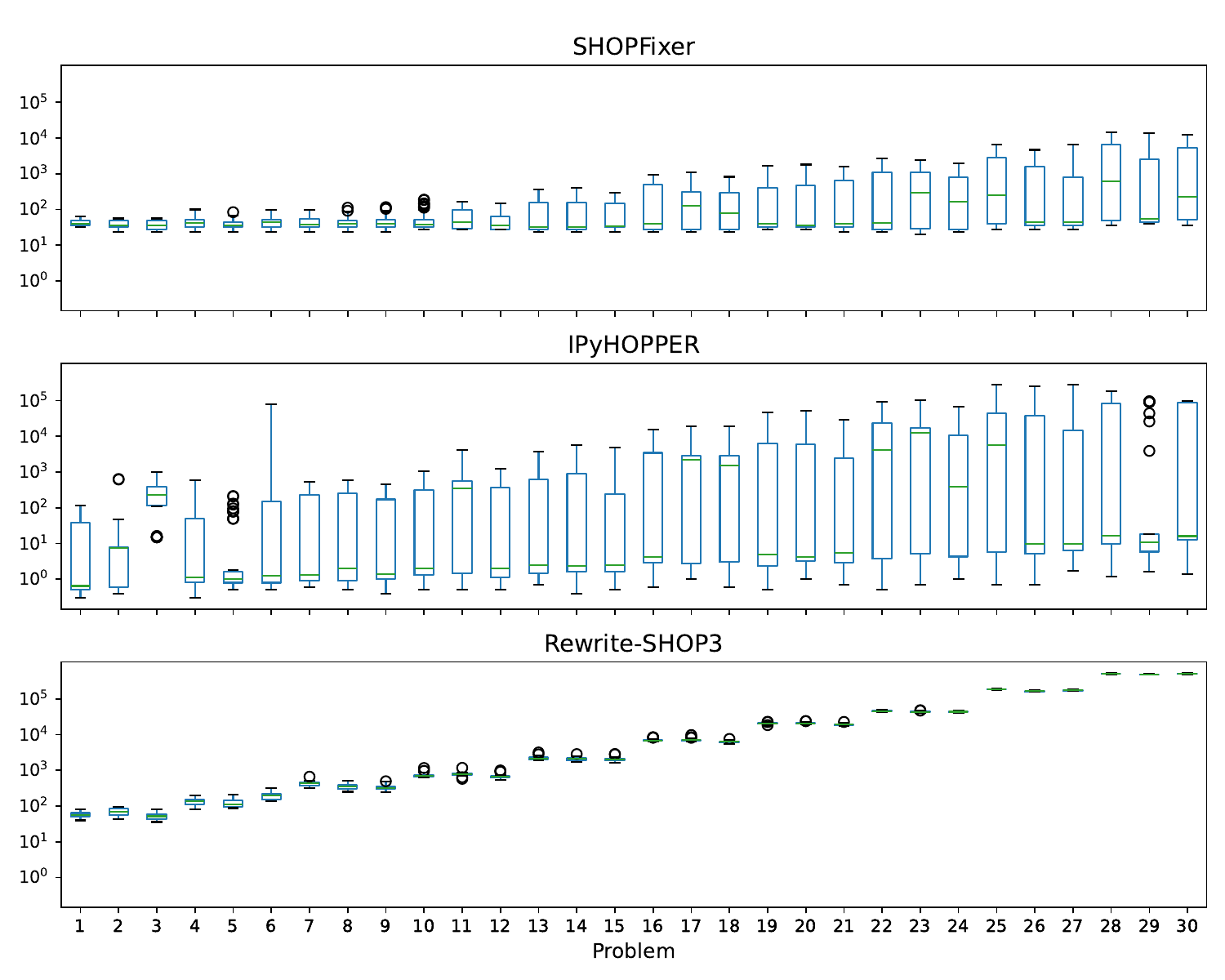}
  \caption{Runtimes for the Openstacks repair problems in $\mathrm{msec}$ (semi-log plot) for each algorithm.  These include only those problems solved
    successfully by the algorithm in question.}
  \label{fig:openstacks-runtimes}
\end{figure}

\begin{figure}[h]
  \centering
  \includegraphics[width=0.9\columnwidth]{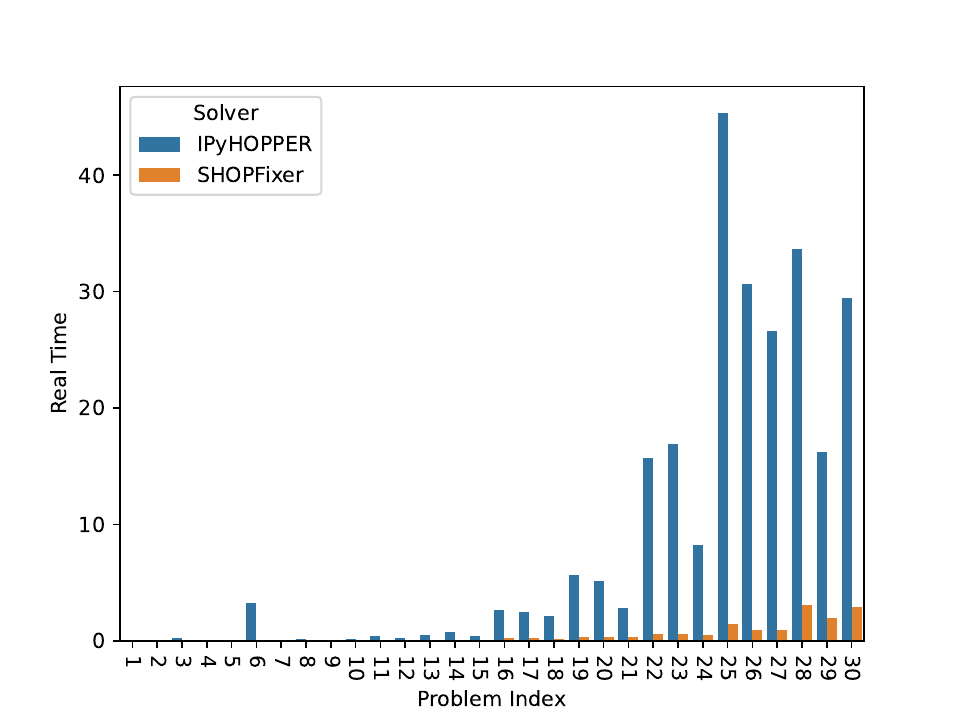}
  \caption{Direct comparison of runtimes for \ipyhopper{} and
    \shopfixer{} on Openstacks problems.  Note that these are
    plotted in seconds, and not on a log scale.}
  \label{fig:openstacks-comparison}
\end{figure}

\begin{algorithm2e}{
\caption{Nondeterministic \sfshort pseudocode}
\label{alg:nondeterministic_shopfixer}
\small
\SetKwFunction{SHOPFixer}{\sfshort}
\SetKwFunction{replace}{replace}
\SetKw{return}{return}
\SetKw{continue}{continue}
\SetKw{break}{break}
\SetKw{None}{None}
\SetKwProg{Fn}{Def}{:}{}

\Fn{\sfshort{ current execution state: $s_c$, unexecuted solution tree: $T_u$, failed action: $a_f$}}{
    $t_r, T_r' \leftarrow$ nondeterministically choose from set of tuples $\{ t_a, T_a$ | $t_a$ is an ancestor of $a_f$, decomposition tree $T_a$ is rooted at $t_a$ and applicable in $ \gamma ( s_c, T_u[ \prec  a_f] \}$ \;
    $T_u' \leftarrow \replace ( T_u, t_r, T_r' )$\;
    \uIf{ $T_r'$ is \None }{
        \return False \;      
    }
    \uElseIf{ ${\color{red}T_u'}$ is applicable in $s_c$ \label{alg:nondeterministic_shopfixer:T_u}}{
        \return $T_u'$ \;   
    }   
    \uElse{
        $a_f' \leftarrow$ first action, $a \in T_u'$, where ${\gamma ( s_c, T_u'[ \prec  a] ) \not\models \color{red}\prestar(a)}$ \; \label{alg:nondeterministic_shopfixer:pre}
        \return \sfshort{$s_c, T_u', a_f'$} \;
    }
}
}\end{algorithm2e}

\input{travel-example.tex}

%% file: rover-full-success-table.tex
\begin{table}[h]
    \begin{adjustbox}{max width=0.98\textwidth}
    \begin{tabular}{l||rrr|rrr|rrr|rrr}
        & \multicolumn{3}{r}{Success Percentage} & \multicolumn{3}{r}{Number Solved} & \multicolumn{3}{r}{Number Failed} & \multicolumn{3}{r}{Proven Unsolvable} \\
       & IPH & RW & SF & IPH & RW & SF & IPH & RW & SF & IPH & RW & SF \\
       Problem &  &  &  &  &  &  &  &  &  &  &  &  \\ \hline \hline
       1 & 100 & 84 & 100 & 50 & 42 & 50 & 0 & 8 & 0 & 0 & 8 & 0 \\
       2 & 100 & 100 & 100 & 50 & 50 & 50 & 0 & 0 & 0 & 0 & 0 & 0 \\
       3 & 64 & 68 & 72 & 32 & 34 & 36 & 18 & 16 & 14 & 18 & 16 & 14 \\
       4 & 100 & 94 & 100 & 50 & 47 & 50 & 0 & 3 & 0 & 0 & 3 & 0 \\
       5 & 100 & 100 & 94 & 50 & 50 & 47 & 0 & 0 & 3 & 0 & 0 & 0 \\
       6 & 80 & 78 & 80 & 40 & 39 & 40 & 10 & 11 & 10 & 0 & 4 & 0 \\
       7 & 100 & 92 & 100 & 50 & 46 & 50 & 0 & 4 & 0 & 0 & 4 & 0 \\
       8 & 100 & 92 & 100 & 50 & 46 & 50 & 0 & 4 & 0 & 0 & 4 & 0 \\
       9 & 100 & 92 & 100 & 50 & 46 & 50 & 0 & 4 & 0 & 0 & 4 & 0 \\
       10 & 74 & 92 & 100 & 37 & 46 & 50 & 13 & 4 & 0 & 0 & 4 & 0 \\
       11 & 100 & 98 & 100 & 50 & 49 & 50 & 0 & 1 & 0 & 0 & 1 & 0 \\
       12 & 100 & 92 & 100 & 50 & 46 & 50 & 0 & 4 & 0 & 0 & 4 & 0 \\
       13 & 96 & 100 & 100 & 48 & 50 & 50 & 2 & 0 & 0 & 0 & 0 & 0 \\
       14 & 100 & 86 & 100 & 50 & 43 & 50 & 0 & 7 & 0 & 0 & 7 & 0 \\
       15 & 100 & 92 & 100 & 50 & 46 & 50 & 0 & 4 & 0 & 0 & 4 & 0 \\
       16 & 100 & 92 & 100 & 50 & 46 & 50 & 0 & 4 & 0 & 0 & 4 & 0 \\
       17 & 100 & 90 & 100 & 50 & 45 & 50 & 0 & 5 & 0 & 0 & 5 & 0 \\
       18 & 100 & 88 & 100 & 50 & 44 & 50 & 0 & 6 & 0 & 0 & 6 & 0 \\
       19 & 100 & 92 & 100 & 50 & 46 & 50 & 0 & 4 & 0 & 0 & 4 & 0 \\
       20 & 36 & 82 & 100 & 18 & 41 & 50 & 32 & 9 & 0 & 0 & 9 & 0 \\ \hline
       Total & 1850 & 1804 & 1946 & 925 & 902 & 973 & 75 & 98 & 27 & 18 & 91 & 14 \\ \hline
       \end{tabular}   
    \end{adjustbox}
    \caption{Full table of successes and failures for Rovers problems.
    Rows where all three methods succeeded are omitted.}
    \label{tab:rover-full-success}
\end{table}

%% file: openstacks-full-success-table.tex
\begin{table*}[h]
  \begin{adjustbox}{max width=0.98\textwidth}
  \begin{tabular}{l||rrr|rrr|rrr|}
    & \multicolumn{3}{r|}{\textbf{Success Percentage}}
    & \multicolumn{3}{r|}{\textbf{Number Solved}}
    & \multicolumn{3}{r|}{\textbf{Number Failed}} \\ 
    Solver & IPyHOPPER & Rewrite-SHOP3 & SHOPFixer & IPyHOPPER &
                                                                 Rewrite-SHOP3
                                       & SHOPFixer & IPyHOPPER &
                                                                 Rewrite-SHOP3 & SHOPFixer \\ 
    Problem Index &  &  &  &  &  &  &  &  &  \\ \hline \hline
7 & 98 & 100 & 100 & 49 & 50 & 50 & 1 & 0 & 0 \\
10 & 98 & 100 & 100 & 49 & 50 & 50 & 1 & 0 & 0 \\
14 & 98 & 100 & 100 & 49 & 50 & 50 & 1 & 0 & 0 \\
19 & 98 & 100 & 100 & 49 & 50 & 50 & 1 & 0 & 0 \\
25 & 94 & 100 & 100 & 47 & 50 & 50 & 3 & 0 & 0 \\
27 & 96 & 100 & 100 & 48 & 50 & 50 & 2 & 0 & 0 \\
28 & 72 & 100 & 100 & 36 & 50 & 50 & 14 & 0 & 0 \\
29 & 78 & 100 & 100 & 39 & 50 & 50 & 11 & 0 & 0 \\
30 & 78 & 100 & 100 & 39 & 50 & 50 & 11 & 0 & 0 \\ \hline
Total &  & & & 1455 & 1500 & 1500 & 45 & 0 & 0 \\

\end{tabular}
\end{adjustbox}

  \caption{Full table of successes and failures for Openstacks problems.
    Rows where all three methods succeeded are omitted.}
  \label{tab:openstacks-full-success-table}
\end{table*}

%% file: satellite-fail-case.tex
\begin{table*}[h]
  \small
\begin{verbatim}
==>
1 (switch_on instrument0 satellite0)
2 (turn_to satellite0 groundstation2 phenomenon6)
3 (calibrate satellite0 instrument0 groundstation2)
4 (turn_to satellite0 phenomenon6 groundstation2) 
 <--- instrument0 becomes decalibrated
5 (take_image satellite0 phenomenon6 instrument0 thermograph0)
6 (turn_to satellite0 star5 phenomenon6)
7 (take_image satellite0 star5 instrument0 thermograph0)
8 (turn_to satellite0 phenomenon4 star5)
9 (take_image satellite0 phenomenon4 instrument0 thermograph0)
root 10
10 (main) -> take-one 11 12
11 (have-image phenomenon6 thermograph0) -> prepare-then-take 13 14
12 (main) -> take-one 15 16
13 (prepare-instrument satellite0 instrument0) -> prepare 17 18
14 (take-image satellite0 instrument0 phenomenon6 thermograph0) -> turn-then-take 4 5
15 (have-image star5 thermograph0) -> prepare-then-take 19 20
16 (main) -> take-one 21 22
17 (turn-on-instrument satellite0 instrument0) -> turn-on 1
18 (calibrate-instrument satellite0 instrument0) -> repoint-then-calibrate 2 3
19 (prepare-instrument satellite0 instrument0) -> prepare 23 24
20 (take-image satellite0 instrument0 star5 thermograph0) -> turn-then-take 6 7
21 (have-image phenomenon4 thermograph0) -> prepare-then-take 25 26
22 (main) -> all-done
23 (turn-on-instrument satellite0 instrument0) -> already-on
24 (calibrate-instrument satellite0 instrument0) -> no-calibration-needed
25 (prepare-instrument satellite0 instrument0) -> prepare 27 28
26 (take-image satellite0 instrument0 phenomenon4 thermograph0) -> turn-then-take 8 9
27 (turn-on-instrument satellite0 instrument0) -> already-on
28 (calibrate-instrument satellite0 instrument0) -> no-calibration-needed
<==

\end{verbatim}
  
  \caption{Original plan, with deviation, for Satellite problem that
    \rewrite{} cannot solve.}
  \label{tab:satellite-unsolvable}
\end{table*}

%% file: travel-example.tex
-------------------------------------------------------------------------------
-------------------------------------------------------------------------------
-------------------------------------------------------------------------------
\subsection{Travel Example}
\label{sec:travel-example}

\begin{figure}[t]
  \centering
  \includegraphics[width=0.55\columnwidth]{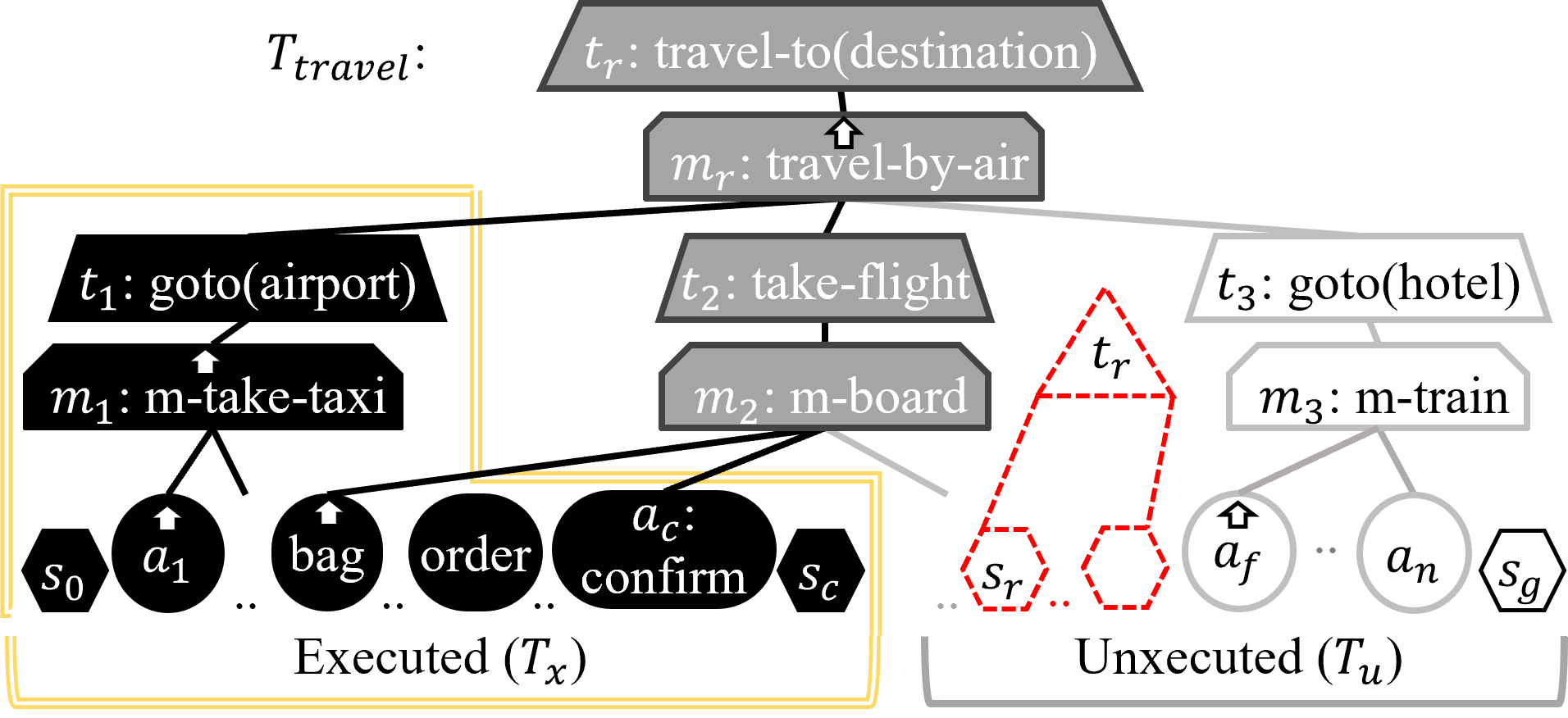}
  \caption{A notional d-tree $T_{travel}$ showing tasks ($t$) as trapezoids, methods ($m$) as boxes, actions ($a$) as circles, and states ($s$) as hexagons; $s_g$ is the goal.  
  The set of executed nodes of the d-tree, denoted $T_x$ (black) are within the amber double line. 
  Partial (gray) and un-executed (white) portions are denoted $T_u$.  
  A failure at the current action $a_c$ results in anomalous state $s_c$, and a repair (red) is introduced at $t_r$.
  A potential failure is expected at action $a_f$.
  }
  \label{fig:travel-example}
\end{figure}

Figure~\ref{fig:travel-example} shows the d-tree $T_{travel}$ for getting a traveler to a destination.  
For indices $i,j,k$, tasks are trapezoids labeled $t_i$, methods are boxes labeled $m_j$, actions are circles labeled $a_k$, and states are hexagons labeled $s_k$.  
Execution for a plan of length $n$ starts in the initial state $s_0$, proceeds through pairs $(a_k, s_k)$  for $1 \leq k \leq n$ until  $s_n \models s_g$.
$T_x$ includes the black nodes of $T_{travel}$, while $T_u$ includes the partially (gray) or un-executed (white) portions. 
Calculating preconditions in $T$ is fully explained in Section~\ref{sec:theory-preliminaries}.  Briefly, for a method or action: 
\textbf{(white arrows)} if it is the first node in its subtree, it intersects its preconditions with the precondition(s) of its parent method(s), possibly up to the root.  
\textbf{(no arrows)} if it is \emph{not} the first, it uses its own precondition.

Figure~\ref{fig:travel-example} shows a solution where method $m_r$ decomposes air travel ($t_r$) into getting to the airport ($t_1$), taking the flight ($t_2$), and getting to the hotel ($t_3$).
$T_{travel}$ uses a taxi for local travel at the originating city ($m_2$) and a train at the destination city ($m_3$).

\paragraph{Making repair necessary.}
Suppose that the traveler has arrived at the airport and checked a bag, denoted by $T_x$.
Deciding to eat before the flight, the traveler is in a restaurant, has ordered food (but not yet eaten or paid) and has just discovered their flight is cancelled during the action $a_c: \texttt{confirm}$. 
The  unexpected current state $s_c$ requires repair for $t_2$.
An ideal solution will be stable in the sense of Fox \etal
Thus, the best repair will maintain commitments, which include: finishing and paying for the meal, leaving the bag checked with the current airline, flying to the same airport, using the train, and lodging in the same hotel. 
We assume that there are suitable methods to book a new flight, recheck a bag, etc.

Repair techniques differ in considering anticipated failures.  
To finish this example, let us also assume that the only flights available at $s_c$ will cause the traveler to miss the last train at the destination; i.e., the precondition of $a_f$ will fail, requiring a repair to $t_3$.